\documentclass[10pt,usletter]{article}

\usepackage[totalwidth=450pt,totalheight=640pt]{geometry}

\usepackage[sort&compress,numbers]{natbib}
\usepackage{pgf}
\usepackage{enumitem}
\usepackage{amsthm,amsmath,amssymb}
\usepackage{tikz}
\usetikzlibrary{shapes.callouts}
\tikzset{
  level/.style   = { ultra thick, blue },
  connect/.style = { dashed, red },
  notice/.style  = { draw, rectangle callout, callout relative pointer={#1} },
  label/.style   = { text width=2cm },
}

\newcommand{\citex}[1]{\citeauthor{#1}~(\citeyear{#1})}

\newcommand{\assgn}[3]{\ensuremath{#1:#2\rightarrow#3}}
\newcommand{\Bassgn}[2]{\ensuremath{#1:#2\rightarrow\{0,1\}}}

\usepackage[backgroundcolor=Moccasin,colorinlistoftodos]{todonotes}

\newcounter{question}

\newcounter{nmcomment}

\usepackage{hyperref}

\newcommand{\hy}{\hbox{-}\nobreak\hskip0pt}

\newcommand{\SB}{\{\,}%
\newcommand{\SM}{\;{:}\;}%
\newcommand{\SE}{\,\}}%
\newcommand{\Card}[1]{|#1|}

\newcommand{\mP}{\text{\normalfont P}}
\newcommand{\NP}{\text{\normalfont NP}}

\newcommand{\W}[1][xxxx]{\text{\normalfont W[#1]}}

\newtheorem{LEM}{Lemma} 
\newtheorem{THE}{Theorem} 
\newtheorem{PRO}{Proposition} 
\newtheorem{COR}{Corollary}

\newtheorem{DEF}{Definition}

\newcommand{\CCC}{\mathcal{C}}
\newcommand{\HHH}{\mathcal{H}}
\newcommand{\AAA}{\mathcal{A}}

\newcommand{\cspi}{I}
\newcommand{\prop}[1]{\mathbb{#1}}  
\newcommand{\arity}{\delta}
\newcommand{\tuple}[1]{\langle{#1}\rangle}  

\newcommand{\concat}{\circ}

\newcommand{\Fun}{\textup{Pol}}

\newcommand{\CON}{\textup{VAL}}
\newcommand{\MAL}{\textup{MAL}}

\newcommand{\MINMAX}{\textup{MINMAX}}

\newcommand{\MAJ}{\textup{MAJ}}
\newcommand{\AFF}{\textup{MIN}}

\newcommand{\KROM}{\textup{2CNF}}
\newcommand{\HORN}{\textup{HORN}}
\newcommand{\AHORN}{\textup{HORN}^-}
\newcommand{\ZVAL}{\textup{0-VAL}}
\newcommand{\OVAL}{\textup{1-VAL}}
\newcommand{\SCHAE}{\textup{Schaefer}}

\newcommand{\STBDET}[1]{\textsc{Strong }$#1$\hy\textsc{Backdoor
    Detection}}
\newcommand{\WEBDET}[1]{\textsc{Weak }$#1$\hy\textsc{Backdoor
    Detection}}

\newcommand{\SATBR}[1]{$#1$\hy\textsc{Strong Branch}}

\newcommand{\defproblem}[3]{
  \vspace{2pt}
\noindent\fbox{
  \begin{minipage}{0.96\columnwidth}
  #1 \\ 
  {\bf{Input:}} #2  \\
  {\bf{Question:}} #3
  \end{minipage}
  }
  \vspace{2pt}
}

\newcommand{\defparproblem}[4]{
  \vspace{2pt}
\noindent\fbox{
  \begin{minipage}{0.96\columnwidth}
  \begin{tabular*}{\textwidth}{@{\extracolsep{\fill}}lr} #1  & {\bf{Parameter:}} #3 \\ \end{tabular*}
  {\bf{Input:}} #2  \\
  {\bf{Question:}} #4
  \end{minipage}
  }
  \vspace{2pt}
}

\newcommand{\shortversion}[1]{}

\begin{document}

\title{Backdoors into Heterogeneous Classes of SAT and CSP}

\author{Serge Gaspers$^1$, Neeldhara Misra$^2$, Sebastian Ordyniak$^3$, \\
  Stefan Szeider$^3$, and Stanislav \v{Z}ivn\'{y}$^4$\\[0.1cm]
\mbox{}\small$^1$UNSW and NICTA, Sydney, Australia (\small sergeg@cse.unsw.edu.au)\\
\mbox{}\small$^2$Indian Institute of Science, Bangalore, India (\small mail@neeldhara.com)\\
\mbox{}\small$^3$TU Wien, Vienna, Austria (\small sordyniak@gmail.com, stefan@szeider.net)\\
\mbox{}\small$^4$University of Oxford, Oxford, UK (\small standa.zivny@cs.ox.ac.uk)
}

\date{}

\maketitle

\begin{abstract}
    In this paper we extend the classical notion of strong and weak backdoor
    sets for SAT and CSP by allowing that different instantiations of the backdoor
    variables result in instances that belong to different base classes;
    the union of the base classes forms a heterogeneous base class.
    Backdoor sets to heterogeneous base classes can be much smaller than
    backdoor sets to homogeneous ones, hence they are much more
    desirable but possibly harder to find.
    We draw a detailed complexity landscape for the problem of detecting
    strong and weak backdoor sets into heterogeneous base classes for SAT and
    CSP.
  \end{abstract}

\section{Introduction}

We consider the propositional satisfiability problem (SAT) and the
constraint satisfaction problem (CSP).
Backdoors are small sets of variables of a SAT or CSP instance that
represent ``clever reasoning shortcuts'' through the search space.
Backdoor sets were originally introduced by Williams, Gomes, and
Selman~(\cite{WilliamsGomesSelman03,WilliamsGomesSelman03a}) as a tool for
the analysis of decision heuristics in propositional
satisfiability. Since then, backdoor sets have been widely used in the
areas of propositional
satisfiability~\cite{WilliamsGomesSelman03,RuanKautzHorvitz04,DilkinaGomesSabharwal07,SamerSzeider08c,KottlerKaufmannSinz08,DilkinaGomesSabharwal09,GaspersSzeider13},
and also for 
material discovery \cite{LeBrasEtal13},
abductive reasoning~\cite{PfandlerRummeleSzeider13},
argumentation~\cite{DvorakOrdyniakSzeider12}, 
quantified Boolean
formulas~\cite{SamerSzeider09a}, and planning~\cite{KroneggerOrdyniakPfandler14,KroneggerOrdyniakPfandler15}.  
A backdoor set is defined with
respect to some fixed \emph{base class} for which the computational
problem under consideration is polynomial-time tractable
(alternatively, it can be defined with respect to a polynomial-time
subsolver). The size of the backdoor set can be seen as a distance
measure that indicates how far the instance is from the target class.
One distinguishes between strong and weak backdoor sets; the latter
applies only to satisfiable instances.
Once a strong backdoor set of size $k$ is
identified, one can decide the satisfiability of the instance by
deciding the satisfiability of at most $d^k$ ``easy'' instances that
belong to the tractable base class, where $d$ denotes the size of the
domain for the variables; for SAT we have $d=2$. Each of the easy
instances is obtained by one of the $d^k$ possible instantiations of
the $k$ variables in the backdoor set. Hence, the satisfiability check
is \emph{fixed-parameter tractable}~\cite{DowneyFellows13} for the combined parameter
backdoor size and domain size ($k+d$). 
A similar approach works for
weak backdoor sets, where the computation of a certificate for
satisfiability (i.e., a satisfying assignment) is fixed-parameter
tractable for the combined parameter ($k+d$).

The fixed-parameter tractability of using the backdoor set for deciding
satisfiability or the computation of a certificate for satisfiability 
triggers the question of whether \emph{finding} a
backdoor set of size at most $k$ is also fixed-parameter tractable. 
In particular, for every base class $\CCC$ one can ask whether the
detection of a strong or weak backdoor set into $\CCC$ of size at
most $k$ is fixed-parameter tractable for parameter~$k$ (possibly in
combination with restrictions on the input or other parameters).  A
systematic study of the parameterized complexity of backdoor set
detection was initiated by~\citet{NishimuraRagdeSzeider04-informal}
for SAT, who showed that the detection of strong backdoor sets into
the classes HORN and 2CNF 
is fixed-parameter tractable,
but the detection of weak
backdoor sets into any of these two classes is $\W[2]$\hy hard (and
hence unlikely to be fixed-parameter tractable).  Since then, the
parameterized complexity of backdoor set detection has become an
active research topic as outlined in a survey
\cite{GaspersSzeider12}.

In this work, we provide two significant extensions to the exciting
research on fixed-parameter tractable backdoor set detection. First,
we introduce \emph{heterogeneous} base classes, which can be composed of 
several homogeneous base classes. We show that heterogeneous base
classes are particularly well-suited for strong backdoor sets, since they
allow that different instantiations of the backdoor variables
result in instances that belong to different base classes, which leads
to arbitrary reductions in the size of a strong backdoor set even
compared to the smallest strong backdoor set for each individual
homogeneous base class. This is in contrast to weak backdoor sets,
where the size of a weak backdoor set into a heterogeneous base class is equal to the
size of the smallest backdoor set into any of the individual base
classes. Here we show that also the complexity of weak backdoor set
detection into heterogeneous base classes is tight to its complexity on
the individual base classes.
Second,  we extend the scope of backdoor set detection from SAT to CSP,
considering target classes that are defined by tractable constraint
languages in terms of certain \emph{closure properties under polymorphisms}.

\paragraph{Heterogeneous Base Classes}
Consider the following SAT instance $F_n=\{C,D_1,\dots,D_n\}$ where
$C= (x \lor \neg a_1 \lor \dots \lor \neg a_{n})$ and $D_i= (\neg x
\lor b_i \lor c_i)$. It is easy to see that any strong backdoor set
into HORN needs to contain at least one of the variables $b_i$ or
$c_i$ from each clause $D_i$, hence such a backdoor set must be of
size $\Omega(n)$; on the other hand, any strong backdoor set into 2CNF
must contain at least $(n-2)$ variables from the clause $C$; hence
such a backdoor must be of size at $\Omega(n)$ as well.  However,
$F_n[x=\text{false}]$ is Horn, and $F_n[x=\text{true}]$ is a 2CNF,
hence the singleton $\{x\}$ constitutes a strong backdoor set into the
\emph{``heterogeneous''} base class $\text{HORN} \cup
\text{2CNF}$. This example shows that by considering heterogeneous
base classes we can access structural properties of instances that are
not accessible by backdoor sets into homogeneous base classes. Identifying a base class
with a class of instances that are solvable by a particular polynomial-time
subsolver, one can consider a heterogeneous base class as a
``portfolio subsolver,'' where for each instance the best suitable
subsolver from the portfolio is chosen.

\paragraph{SAT Backdoor Sets}
A natural question at this point is whether the fixed-parameter
tractability results for the detection of strong backdoor sets into
individual base classes can be extended to more powerful heterogeneous 
base classes. In this work, we completely characterize the complexity 
landscape for heterogeneous base classes obtained arbitrary combinations of the 
well-known Schaefer classes~\cite{Schaefer78}, in the following denoted
by $\HORN$ (for Horn formulas), $\AHORN$ (for Anti-Horn formulas),
$\KROM$ (for 2-CNF formulas), $\ZVAL$ (for zero valid formulas), 
and $\OVAL$ (for one valid formulas). 
To state the classification, we briefly 
introduce some terminology. We say that a pair $(\CCC,\CCC')$ of Schaefer 
classes is a \emph{bad pair} if $\CCC \in \{\HORN,\ZVAL\}$ and $\CCC' \in
\{\AHORN,\OVAL\}$. 

Let $\CCC$ be a class of CNF-formulas such that $\CCC=\bigcup_{s \in S}s$ 
for some subset $S$ of the Schaefer classes. Our main result for SAT
backdoor sets (Theorem~\ref{the:sat-dico})
is that \STBDET{\CCC} is fixed parameter tractable if and only if $S$ does not contain a bad pair of Schaefer classes, otherwise it is $\W[2]$-hard. 
On the other hand, detecting weak backdoors is always $\W[2]$-hard 
for any choice of $\CCC$, and we show this by building on the known
hardness results for when $\CCC$ is a singleton set (Theorem~\ref{wbd:w-hard}). 

We also show that \STBDET{\CCC} as well as \WEBDET{\CCC} are fixed-parameter tractable for the 
combined parameter $k$ and the maximum length $r$ of a clause of the input 
formula (Theorem~\ref{the:sat-clause-sb} and~\ref{the:sat-clause-wb}). These 
algorithms formalize the intuition that all the hardness results in the 
previous classification result exploit clauses of unbounded length.

\paragraph{CSP Backdoor Sets}
The identification of tractable classes of CSP instances has been
subject of extensive studies. A prominent line of research, initiated
by \citet{Schaefer78} in his seminal paper on Boolean CSP, is to
identify tractable classes by restricting the relations that may
appear in constraints to a prescribed set, a constraint language.
Today, many constraint languages have been identified that give rise
to tractable classes of CSPs \cite{BulatovDalmau06,Idziak10:siam,Barto14:jacm};
typically such languages are defined in terms of certain \emph{closure
  properties}, which ensure that the relations are closed under
pointwise application of certain operations on the domain. For
instance, consider a CSP instance whose relations are closed under a
constant operation $f(x) = d$ for some $d \in D$ (such a operation is a
polymorphism of the relations in the instance). 
Then note that every relation is either empty or forced
to contain the tuple $\langle d,d,\ldots,d \rangle$. Thus, given a
particular instance, we may either declare it unsatisfiable (if it
contains a constraint over the empty relation), or satisfy it
trivially by setting every variable to $d$. Further examples of
polymorphisms for which closure properties yield tractable CSP are
min, max, majority, affine, and Mal'cev operations~\cite{JeavonsCohenGyssens97,Barto14:survey}.
 
We study the problem of finding strong and weak backdoor sets into tractable classes
of CSP instances defined by certain polymorphisms.  Our main result
for CSP backdoors (Theorem~\ref{the:fpt-csp}) establishes
fixed-parameter tractability for a wide range of such base classes.
In particular, we show that the detection of strong and weak backdoor sets is
fixed-parameter tractable for the combined parameter backdoor size,
domain size, and the maximum arity of constraints.  In fact, this
result entails \emph{heterogeneous} base classes, as different
instantiations of the backdoor variables can lead to reduced instances
that are closed under different polymorphisms (even polymorphisms of
different type). We complement our main result with hardness results
that show that we lose fixed-parameter tractability when we
omit either domain size or the maximum arity of constraints from
the parameter (Theorems~\ref{the:domain} and \ref{the-arity2}). Hence,
Theorem~\ref{the:fpt-csp} is tight in a certain sense.

\paragraph{Related Work}

Recently, two papers dealing with strong backdoor set detection for the
constraint satisfaction problem have
appeared~\cite{CarbonnelCooperHebrard14,GanianRamanunjanSzeider16},
which nicely supplement (however are mostly orthogonal to) the results
in this paper. \citet{CarbonnelCooperHebrard14} show
$\W[2]$-hardness for strong backdoor set detection parameterized by
the size of the backdoor set even for CSP-instances with only one
constraint (however with unbounded domain and unbounded arity). They also give
a fixed-parameter algorithm for strong backdoor set detection
parameterized by the
size of the backdoor set and the maximum arity of
any constraint, if the base class is ``h-Helly'' for any fixed integer
$h$.
However, as is also mentioned in~\citet{CarbonnelCooper15}, the
``h-Helly'' property is rather restrictive and orthogonal to our
approach of considering the domain as an additional parameter.

\citet{GanianRamanunjanSzeider16} show fixed-parameter
tractability of strong backdoor detection parameterized by the size of
the backdoor to a very general family of heterogeneous base classes,
i.e., base classes defined as all CSP-instances obtained from the
disjoint union of tractable, finite, and semi-conservative constraint languages. 
These base classes are orthogonal to the base classes considered in
this paper. They are more general in the sense that they allow for
the CSP-instance to consist of pairwise disjoint subinstances each
belonging to a different
tractable class, and they are more specific in the sense that they
only consider finite and semi-conservative constraint languages, which
also restricts the CSP-instances to bounded domain and bounded arity. 

\section{Preliminaries}

\paragraph{SAT}

A \emph{literal} is a propositional variable $x$ or a negated variable
$\neg x$. We also use the notation $x=x^1$ and $\neg x= x^0$. A
\emph{clause} is a finite set of literals that does not contain a
complementary pair $x$ and $\neg x$. A propositional formula in
conjunctive normal form, or \emph{CNF formula} for short, is a set of
clauses. For a clause $C$ we write $\mbox{var}(C)=\SB x \SM x\in C$ or $\neg x \in C\SE$, and for a CNF
formula $F$ we write $\mbox{var}(F)=\bigcup_{C\in F} \mbox{var}(C)$.  

For a set $X$ of propositional variables we denote by $X\rightarrow\{0,1\}$ the set of
all mappings from $X$ to $\{0,1\}$, the truth assignments on $X$.
We denote by $\overline{X}$ the set of literals corresponding to the negated variables of $X$.
For $\Bassgn{\tau}{X}$ we let 
$\text{true}(\tau)=\SB x^{\tau(x)}\SM x\in X\SE$ and
$\text{false}(\tau)=\SB x^{1-\tau(x)}\SM x\in X \SE$ be the sets of literals set
by $\tau$ to $1$ and $0$, respectively. 
Given a CNF formula $F$ and a truth assignment $\Bassgn{\tau}{X}$ we define
$F[\tau]=\SB C\setminus \text{false}(\tau) \SM C\in F,\ C\cap\text{true}(\tau)=
\emptyset\,\}$. If $\Bassgn{\tau}{\{x\}}$ and $\epsilon=\tau(x)$, we simply
write $F[x=\epsilon]$ instead of $F[\tau]$.

A CNF formula $F$ is \emph{satisfiable} if there is some
$\Bassgn{\tau}{\mbox{var}(F)}$ with $F[\tau]=\emptyset$, otherwise $F$ is
\emph{unsatisfiable}.

\paragraph{CSP}

Let $D$ be a set and $n$ and $n'$ be natural numbers. 
An $n$-ary relation on $D$
is a subset of $D^n$. For a tuple $t \in D^n$, we denote by $t[i]$,
the $i$-th entry of $t$, where $1 \leq i \leq n$. For
two tuples $t \in D^n$ and $t' \in D^{n'}$, we denote by $t \concat
t'$, the concatenation of $t$ and $t'$.

An instance of a \emph{constraint satisfaction problem} (CSP) $\cspi$
is a triple $\tuple{V,D,C}$, where $V$ is a finite set of variables
over a finite set (domain) $D$, and $C$ is a set of constraints. A
\emph{constraint} $c \in C$ consists of a \emph{scope},
denoted by $V(c)$, which is an ordered list of a subset of
$V$, and a relation, denoted by $R(c)$, which is a
$|V(c)|$-ary relation on $D$. To simplify notation, we sometimes treat
ordered lists without repetitions, such as the scope of a
constraint, like sets. 
For a variable $v \in V(c)$ and a tuple $t \in R(c)$, we denote
by $t[v]$, the $i$-th entry of $t$, where $i$ is the position of $v$ in
$V(c)$.
For a CSP instance $\cspi=\tuple{V,D,C}$ we sometimes denote by $V(\cspi)$,
$D(\cspi)$, $C(\cspi)$, and $\arity(\cspi)$, its set of variables $V$, its domain $D$,
its set of constraints $C$, and the maximum arity of any constraint of
$\cspi$, respectively. 

Let $V' \subseteq V$ and $\assgn{\tau}{V'}{D}$. For
a constraint $c \in C$, we denote by $c[\tau]$, the constraint whose 
scope is $V(c) \setminus V'$ and whose relation contains all
$|V(c[\tau])|$-ary tuples $t$ such that there is a $|V(c)|$-ary
tuple $t' \in R(c)$ with $t[v]=t'[v]$ for every $v \in V(c[\tau])$ and
$t'[v]=\tau(v)$ for every $v \in V'$. We denote by $\cspi[\tau]$ the
CSP instance with variables $V \setminus V'$, domain $D$, and
constraints $C[\tau]$, where $C[\tau]$ contains a constraint $c[\tau]$
for every $c \in C$.

A \emph{solution} to a CSP instance $\cspi$ is a mapping $\assgn{\tau}{V}{D}$
such that $\tuple{\tau[v_1],\dotsc,\tau[v_{|V(c)|}]}
\in R(c)$ for every $c \in C$ with $V(c)=\tuple{v_1,\dotsc,v_{|V(c)|}}$.

\paragraph{Backdoors}

Backdoors are defined relative to some fixed class $\CCC$ of instances
of the problem under consideration (i.e., SAT or CSP). One usually
assumes that the problem is tractable for instances from $\CCC$, as
well as that the recognition of $\CCC$ is tractable.
 
In the context of SAT, we define a \emph{strong} \emph{$\CCC$-backdoor
  set}  of a CNF formula $F$ to be  a set $B$ of variables such that
$F[\tau]\in \CCC$ for each $\Bassgn{\tau}{B}$. 
A \emph{weak}
\emph{$\CCC$-backdoor set} of $F$ is a set $B$ of variables such that
$F[\tau]$ is satisfiable and $F[\tau]\in \CCC$ holds for some
$\Bassgn{\tau}{B}$.
If we know a strong $\CCC$-backdoor set of $F$, we can decide the
satisfiability of $F$ by checking the satisfiability of $2^k$ ``easy''
formulas $F[\tau]$ that belong to $\CCC$. Thus SAT decision is
fixed-parameter tractable in the size $k$ of the backdoor.
Similarly, in the context of CSP, we define a \emph{strong}
\emph{$\CCC$-backdoor set} of a CSP instance $\cspi=\tuple{V,D,C}$ as
a set $B$ of variables such that $\cspi[\tau]\in \CCC$ for every
$\assgn{\tau}{B}{D}$. We also call a strong $\CCC$\hy backdoor a
\emph{strong backdoor set into $\CCC$}.
  A \emph{weak} \emph{$\CCC$-backdoor set} of
  $\cspi$ is a set $B$ of variables such that $\cspi[\tau]$ is
  satisfiable and $\cspi[\tau]\in \CCC$ holds for some $\assgn{\tau}{B}{D}$.  If we know a strong $\CCC$-backdoor set of $\cspi$
of size $k$, we can reduce the satisfiability of $\cspi$ to the
satisfiability of $d^k$ CSP instances in $\CCC$ where $d=\Card{D}$.
Thus deciding the satisfiability of a CSP instance is fixed-parameter
tractable in the combined parameter $d+k$.

The challenging problem is---for SAT and for CSP---to find a
strong, or weak $\CCC$\hy backdoor set of size at most $k$ if one
exists. 

For each class $\CCC$ of SAT or CSP instances, we define the following
problem.

\defproblem{\STBDET{\CCC}}{A SAT or CSP instance $\cspi$ and a
  nonnegative integer $k$.}{Does $\cspi$ have a strong $\CCC$-backdoor
  set of size at most $k$?}

The problem \WEBDET{\CCC} is defined similarly.

\paragraph{Parameterized Complexity}
We provide basic definitions of parameterized complexity; for an
in-depth treatment we refer to the recent monograph
\cite{DowneyFellows13}.  A problem is \emph{parameterized} if each
problem instance $I$ is associated with a nonnegative integer $k$, the
parameter.  A parameterized problem is \emph{fixed-parameter
  tractable} (or \emph{FPT}, for short) if there is an algorithm, $A$,
a constant $c$, and a computable function $f$, such that $A$ solves
instances of input size $n$ and parameter $k$ in time $f(k)n^c$. 
Fixed-parameter tractability extends the conventional notion of
polynomial-time tractability, the latter being the special case where
$f$ is a polynomial.  The so-called Weft-hierarchy $\W[1]\subseteq
\W[2] \subseteq \dots$ contains classes of parameterized decision
problems that are presumed to be larger than FPT. It is believed that problems that are hard for any of the classes in the Weft-hierarchy are not fixed-parameter tractable.  The
classes are closed under \emph{fpt-reductions} that are
fixed-parameter tractable many-one reductions, which map an instance
$x$ with parameter $k$ of one problem to a decision-equivalent
instance $x'$ with parameter $k'$ of another problem, where $k'\leq
f(k)$ for some computable function $f$.

For instance, the following problem is well-known to be $\W[2]$\hy
complete \cite{DowneyFellows13}.

\defparproblem{{\sc Hitting Set}}{ A set system ${\mathcal Q}$ over a
set (universe) $U$, and a positive integer $k$}{$k$}{Is there a
subset $X \subseteq U$, with $|X|\le k$ such that every
set $S \in {\mathcal Q}$ contains at least one element from $X$?}

\section{Backdoor Detection for SAT}

\newcommand{\HHTT}{\text{H or 2}}
\newcommand{\HH}{\text{H}}
\newcommand{\TT}{\text{T}}

Schaefer's base classes \cite{Schaefer78} give rise to classes of CNF
formulas defined in terms of syntactical properties of clauses.\footnote{Affine
Boolean formulas considered by Schaefer do not correspond naturally to a
class of CNF formulas, hence we do not consider them here. However, we do consider the
affine case in the context of Boolean CSPs; cf.~Theorem~\ref{the:domain}.}
A clause is
\begin{itemize}
 \item \emph{Horn} if it contains at most one positive literal,
 \item \emph{Anti-Horn} if it contains at most one negative literal,
 \item \emph{2CNF} if it contains at most two literals,\footnote{A clause
 containing exactly two literals is also known as a \emph{Krom}
 clause~\cite{Knuth11:taocp4a}.}
 \item \emph{0-valid} if it is empty or contains at least one negative literal, and
 \item \emph{1-valid} if it is empty or contains at least one positive literal\footnote{0-valid and 1-valid is often defined more restrictively, not allowing empty clauses.}.
\end{itemize}
A CNF formula is \emph{Horn}, \emph{Anti-Horn}, etc.\ if it contains only Horn, Anti-Horn, etc.\
clauses. We denote the respective classes of CNF formulas by HORN,
HORN$^-$, 2CNF, 0-VAL, and 1-VAL.

\subsection{Strong Backdoor Sets}

In this section we study the parameterized complexity of
\STBDET{\CCC}, where $\CCC$ is the union of any subset of the Schaefer
classes. In the case that $\CCC$ consists only of a single Schaefer
class this has been intensively studied
before~\cite{NishimuraRagdeSzeider04-informal,GaspersSzeider12} and it
is known that \STBDET{\CCC} is polynomial for 0-VAL and 1-VAL,
and FPT for the remaining Schaefer classes.

The results in this subsection are as follows. First we show that \STBDET{\HHH} is fixed-parameter tractable parameterized by the size of the backdoor set for various combinations of the Schaefer classes, i.e., the case that $\HHH=\KROM \cup \CCC$ for any $\CCC \in \SCHAE$ is shown in Theorem~\ref{the:sat-fpt-krom}, the cases that $\HHH=\HORN \cup \ZVAL$ and
$\HHH=\AHORN \cup \OVAL$ are shown in Theorem~\ref{the:sat-fpt-horn-zval}, and finally 
Theorem~\ref{the:sat-fpt-triple} shows tractability for the cases that
$\HHH=\KROM \cup \HORN \cup \ZVAL$ and $\HHH=\KROM \cup \AHORN \cup \OVAL$. 
In contrast, Theorem~\ref{the:sat-bp-w2} captures a condition for hardness of the
strong backdoor detection problem, which is complementary to all the
cases in the first three results. Together, and combined with known
results in the literature, these theorems give us a complete
complexity classification, which is summarized in Theorem~\ref{the:sat-dico}. 
  
Our hardness results rely on the fact that the
clauses have unbounded length. Theorem~\ref{the:sat-clause-sb} shows
that if we consider the maximum length of any clause as an additional
parameter, then the problem of finding a strong heterogeneous backdoor
becomes fixed-parameter tractable for any combination of the Schaefer classes. 

As mentioned above we employ branching algorithms for the detection of
strong backdoor sets. In particular, all algorithms are based on a
depth-bounded search tree approach and mainly differ in the set of
employed branching rules. Informally, the algorithms construct a
search tree, where each node is labeled with a potential backdoor set
and the crucial difference between the algorithms lies in
chosing the set of possible extensions of the potential backdoor set
on which to branch on. In this way backdoor set detection can be
reduced to the following problem, which given a potential backdoor set $B'$
computes the set $\mathcal{Q}$ of potential extensions of $B'$, which can be
branched on in the next step.

\defproblem{\SATBR{\CCC}}{A CNF formula $F$ with variables $V$, a non-negative integer $k$, and a subset $B' \subseteq V$ with $|B'|\leq k$.}
{Output:
  \begin{itemize}
  \item \textsc{Yes} iff $B'$ is a strong $\CCC$\hy backdoor set of
    $F$, or
  \item a family $\mathcal{Q}$ of subsets of $V \setminus B'$ such that:
    \begin{itemize}
    \item[B1] for every $Q \in \mathcal{Q}$ it holds that $|B' \cup Q|\leq k$, and
    \item[B2] for every strong $\CCC$\hy backdoor set $B$ of $F$ with
      $B' \subseteq B$ and $|B|\leq k$ there is a $Q
      \in \mathcal{Q}$ such that $B'\cup Q \subseteq B$.
    \end{itemize}
  \end{itemize}
}

The next theorem shows that \STBDET{\CCC} can be reduced to
\SATBR{\CCC} via a branching algorithm that uses an algorithm $\AAA$
for \SATBR{\CCC} as a subroutine to determine the set of braching
choices. Importantly the running time of the resulting branching
algorithm for \STBDET{\CCC} crucially depends on the ``type'' of
braching rules returned by $\AAA$, or in other words the branching
factor. To be able to quantify the branching factor we introduce the
function $T_\AAA(k)$ defined recursively as 

\[T_\AAA(k)=\max_{\mathcal{Q} \in
  \mathcal{O}}\left (|\mathcal{Q}|T_\AAA(k-\min_{Q\in \mathcal{Q}}|Q|)\right)\textup{ and }T_\AAA(0)=0\]

where $\mathcal{O}$ is the set of all
possible output sets $\mathcal{Q}$ of $\AAA$ on any input $F$,
$B' \subseteq V$ with $|B'|\leq k$, and $k$.
\begin{THE}\label{the:sat-det}
  Let $\CCC$ be a class of CNF formulas, $F$ be a CNF formula, $k$
  a non-negative integer, and let $\AAA$ be an algorithm
  solving \SATBR{\CCC}. Then \STBDET{\CCC} can be solved in time
  $O(R_\AAA(F,k) \cdot T_\AAA(k))$, where $R_\AAA(F,k)$ is the time required by
  $\AAA$ on any input $F$, $B' \subseteq V$ with $|B'|\leq k$, and $k$.
\end{THE}
\begin{proof}
  We will employ a branching algorithm that employs the algorithm
  $\AAA$ solving \SATBR{\CCC} as a subroutine. 
  The algorithm uses a depth-bounded
  search tree approach to find a strong $\CCC$-backdoor set 
  of size at most $k$. Let $F$ be any CNF-formula with variables $V$.

  We construct a search tree $T$, for which
  every node is labeled by a set $B$ of at most $k$ variables of
  $V$. Additionally, every leaf node has a second label,
  which is either \textsc{Yes} or \textsc{No}. 
  $T$ is defined inductively as follows. The root of $T$ is labeled by
  the empty set. Furthermore, if $t$ is a node of $T$,
  whose first label is $B'$, then the
  children of $t$ in $T$ are obtained as follows. 
  First we run the algorithm $\AAA$ on $F$, $B'$, and
  $k$. If $\AAA$ returns \textsc{Yes} then $t$ becomes a leaf
  node, whose second label is \textsc{Yes}. Otherwise let $\mathcal{Q}$ be the
  set of subsets of $V \setminus B'$ returned by $\AAA$. If
  $\mathcal{Q}=\emptyset$, then there is no strong $\CCC$\hy backdoor set $B$ of
  $F$ with $B' \subseteq B$ and $|B|\leq k$ and consequently $t$
  becomes a leaf node whose second label is \textsc{No}. Otherwise,
  i.e., if $\mathcal{Q} \neq \emptyset$ then $t$ has a a child for every $Q \in
  \mathcal{Q}$, whose first label is $B' \cup Q$. This completes the definition of $T$.
  If $T$ has a leaf node, whose second label is \textsc{Yes}, then the
  algorithm returns the first label of that leaf node as a strong
  $\CCC$\hy backdoor set of $F$. Otherwise the
  algorithm returns \textsc{No}. This completes the description of the
  algorithm.

  We now show the correctness of the algorithm. First, suppose 
  the search tree $T$ built by the algorithm has a leaf node $t$ whose
  second label is \textsc{Yes}. Here, the algorithm returns the first label, say
  $B$ of $t$. By definition of $T$, we obtain that $|B|\leq k$ and $B$
  is a strong $\CCC$\hy backdoor set of $F$.

  Now consider the case where the algorithm returns
  \textsc{No}. We need to show that $F$ has no strong $\CCC$\hy
  backdoor set of size at most $k$.
  Assume, for the sake of contradiction that such a set $B$ exists.

  Observe first that for every leaf node $t$ of $T$ whose second label
  is \textsc{No}, it holds that $B' \nsubseteq B$. This is because
  the algorithm $\AAA$ on input $F$, $B'$, and $k$ returned the
  empty set for $\mathcal{Q}$. Hence if $t$ is a leaf node whose first label
  $B'$ satisfies $B' \subseteq B$, then the second label of $t$ must
  be \textsc{Yes}.

  It hence remains to show that $T$ has a leaf node whose first label is a set $B'$
  with $B' \subseteq B$. This will complete the proof about the
  correctness of the algorithm. We will show a slightly stronger
  statement, namely, that for every natural number $\ell$, either $T$
  has a leaf whose first label is contained in
  $B$ or $T$ has an inner node of distance exactly $\ell$ from the root
  whose first label is contained in $B$. We show the latter by
  induction on $\ell$.

  The claim obviously holds for $\ell=0$. So assume that $T$ contains a
  node $t$ at distance $\ell$ from the root of $T$ whose first label, say
  $B'$, is a subset of $B$. 
  If $t$ is a leaf node of $T$, then the
  claim is shown. Otherwise, the algorithm $\AAA$
  determined that $B'$ is not a strong $\CCC$\hy backdoor set of $F$ and
  returned a set $\mathcal{Q}$ of subsets of $V\setminus B'$ satisfying B2.
  Hence there is a $Q \in \mathcal{Q}$ with $B' \cup Q \subseteq B$ and because
  $t$ contains a child $t'$ whose first label is $B' \cup Q$, we
  obtain that $t'$ is a node at distance $\ell+1$ satisfying the
  induction invariant.
  This concludes the proof concerning the correctness of the algorithm.

  The running time of the algorithm is obtained as
  follows. Let $T$ be a search tree obtained by the
  algorithm. Then the running time of the depth-bounded search tree
  algorithm is $O(|V(T)|)$ times the maximum time that is spend on any
  node of $T$. Since $|V(T)| \leq T_\AAA(k)$ and the time spend on any node
  is at most $R_\AAA(F,k)$, the stated running time follows.
\end{proof}

Our first tractability result is for \STBDET{(\KROM \cup \CCC)}, where $\CCC \in \SCHAE$. Before formally stating the algorithm, we provide a brief intuition for the branching rules. Typically, we branch on an obstructing clause (or a pair of obstructing clauses), 
a situation that occurs in one of two flavors:

\begin{enumerate}
\item[(1)] Either there is a clause $C$ that is not in $\KROM \cup \CCC$, or,
\item[(2)] there is a pair of clauses $C$ and $C'$, where $C \in \KROM \setminus \CCC$ 
  and $C' \in \CCC \setminus \KROM$.
\end{enumerate}

Consider the case when $\CCC = \ZVAL$. In scenario (1), the formula
has a clause $C$ of length at least three with only positive
literals. Because $C$ has only positive literals, it follows that
no subclause of $C$ is $0$-valid. Hence every backdoor set has to hit all but 
at most two literals from $C$. It follows that if $C$ has more than $k+2$ literals, then there is
no strong backdoor of size $k$ to $\ZVAL \cup \KROM$. Otherwise, the
clause has at most $k+2$ literals on which we can branch
exhaustively. In scenario (2), we know that $C$ has at most two
literals and because $C \notin \ZVAL$ we obtain that
any strong backdoor set that does not contain a variable
from $C$ must reduce $C'$ to a clause with at most two
literals. Hence, any strong backdoor set has to contain either a
variable from $C$ (but there are at most two of those) or must contain 
a variable from any subset of three variables in $C'$.

Consider the case when $\CCC = \HORN$. In scenario (1) above, the
formula has a clause $C$ of length
at least three with at least two positive literals, say $x$ and
$y$. Let $\ell$ be any other literal in $C$. It is clear that any
strong backdoor set must contain one of $x$, $y$ or $\ell$, and we branch
accordingly. In scenario (2), we know that $C$ has at most two
literals and $C'$ contains at least two positive literals. Because
every strong backdoor set that does not contain a variable
from $C$ must reduce $C'$ to a clause with at most two
literals, we can branch on the variables in $C$ and any subset of
three variables in $C'$.

The intuitive explanations for $\OVAL$ and $\AHORN$ are
analogous. Note that the criteria described above need to be refined
further at internal nodes of the branching process. We now turn to a
formal exposition of the algorithm.

\begin{THE}\label{the:sat-fpt-krom}
  Let $\CCC \in \SCHAE$. Then, \STBDET{(\KROM \cup \CCC)}
  is fixed-parameter tractable parameterized by the size of the backdoor set.
\end{THE}
\begin{proof}
  Let $\HHH=\KROM \cup \CCC$. We will provide an algorithm
  $\AAA$ solving \SATBR{\HHH} in time $O(2^k|F|)$ and
  satisfying $T_\AAA(k)=9^k$. The result then follows from Theorem~\ref{the:sat-det}.

  Let $F$ be the given CNF formula with variables $V$, $k$ the given
  integer, and $B' \subseteq V$ with $|B'|\leq k$. The algorithm
  $\AAA$ first checks whether $F[\tau] \in \HHH$ for every
  assignment $\Bassgn{\tau}{B'}$. If that is the case then
  $\AAA$ returns \textsc{Yes}. Otherwise let 
  $\Bassgn{\tau}{B'}$ be an assignment such that
  $F[\tau] \notin \HHH$, we consider two cases: 
  \begin{enumerate}
  \item if $|B'|=k$, then there is no strong $\HHH$\hy backdoor set
    $B$ of
    $F$ with $B' \subseteq B$ and $|B|\leq k$ and the algorithm
    $\AAA$ correctly returns the set $\mathcal{Q}=\emptyset$,
  \item otherwise, i.e., if $|B'|<k$, then
    we again distinguish two cases:
    \begin{enumerate}
    \item there is a clause $C \in F[\tau]$ with $C \notin \HHH$, 
      then we again distinguish two cases:
      \begin{enumerate}
      \item if $\CCC \in \{\HORN,\AHORN\}$, then let $O$
        be a subset of $C$ with $|O|=3$ and $O \notin \CCC$.
        Then $\AAA$ returns the set $\mathcal{Q}=\SB \{v\} \SM v \in
        \mbox{var}(O)\SE$.

        \medskip
        To show the correctness of this step assume that there is a
        strong $\HHH$\hy backdoor set $B$ of $F$ with $B' \subseteq B$
        and $|B|\leq k$. It suffices to show that $B \cap O\neq
        \emptyset$. Suppose for a contradiction that this is not the
        case and let $\Bassgn{\tau^*}{B}$ be
        an assignment compatible with $\tau$ that does not satisfy any
        literal of $C$. Then, $F[\tau^*]$ contains a superset of $O$ as a
        clause and hence $F[\tau^*] \notin \KROM \cup \CCC$,
        contradicting our assumption that $B$ is a strong $\HHH$\hy
        backdoor set of $F$.
      \item if $\CCC \in \{\OVAL,\ZVAL\}$, then we again distinguish
        two cases:
        \begin{enumerate}
        \item $|B'|+|C|-2>k$, then the algorithm $\AAA$
          returns the set $\mathcal{Q}=\emptyset$.
          
          \medskip
          The correctness of this step follows from the fact that
          every strong $\HHH$\hy backdoor set $B$ of $F$ with
          $B'\subseteq B$ has to contain all but at most two variables
          of $C$ and hence $B$ would have size larger than $k$.
        \item otherwise, i.e., if $|B|+|C|-2 \leq k$, then
          the algorithm $\AAA$ returns the set $\mathcal{Q}$ containing
          all sets $\mbox{var}(C\setminus O)$
          for every two element subset $O$ of $C$.

          \medskip
          To show the correctness of this step assume that there is a
          strong $\HHH$\hy backdoor set $B$ of $F$ with $B' \subseteq B$
          and $|B|\leq k$. It remains to show that $|\mbox{var}(C) \setminus B|\leq
          2$. Assume not, and let $\Bassgn{\tau^*}{B}$ be an assignment
          compatible with $\tau$ that does not satisfy any literal of $C$. 
          Then, $C[\tau^*] \in F[\tau^*]$ and because $|C[\tau^*]|>2$ and 
          $C[\tau^*] \notin \CCC$, also $F[\tau^*] \notin \KROM
          \cup \CCC$, contradicting our assumption that $B$ is a strong $\HHH$\hy
          backdoor set of $F$.
        \end{enumerate}
      \end{enumerate}
    \item otherwise, there is a clause $C$ in $F[\tau]$ with $C \in
      \KROM \setminus \CCC$ and a clause $C'$ in $F[\tau]$ with $C' \in
      \CCC \setminus \KROM$. Let $O$ be the set of variables
      containing the (at most) two variables from $C$ plus any three
      element subset of variables from $C'$. Then, the algorithm
      $\AAA$ returns the set $\mathcal{Q}=\SB \{v\}\SM v \in O\SE$.

      \medskip
      To show the correctness of this step assume that there is a
      strong $\HHH$\hy backdoor set $B$ of $F$ with $B' \subseteq B$
      and $|B|\leq k$.
      It remains to show that $B \cap
      \mbox{var}(O)\neq \emptyset$. Assume not, and let $\Bassgn{\tau^*}{B}$ be
      an assignment compatible with $\tau$ that does not satisfy any
      literal of $C'$. Then, $F[\tau^*]$ contains the clause $C$
      together with a subclause $C''$ of $C'$ containing at least three
      literals. Because $C \notin \CCC$ and $C'' \notin \KROM$, we
      obtain that $F[\tau^*] \notin \KROM \cup \CCC$.
    \end{enumerate}
  \end{enumerate}
  This completes the description and the proof of correctness of the
  algorithm $\AAA$. The running time of $\AAA$ is the
  number of assignments of the at most $k$ variables in $B'$ times the
  time required to test whether the reduced formula is in $\HHH$,
  i.e., $\AAA$ runs in time $O(2^k|F|)$. It remains to obtaine
  the stated bound of $9^k$ for the function $T_\AAA(k)$.
  According to the three branching rules given in 
  (2.a.i), (2.a.ii.B), and (2.b) we obtain that $T_\AAA(k)$ can be
  bounded by the maximum of the following recurrence relations:
  (1) $T_\AAA(k)=3T_\AAA(k-1)$ (2.a.i), (2) $T_\AAA(k)=r^2T_\AAA(k-(r-2))$ for any $r>2$ (2.a.ii.B), 
  and (3) $T_\AAA(k)=5T_\AAA(k-1)$ (2.b). In case of (1) $T_\AAA(k)$ is at most 
  $3^k$, in case of (2) the maximum of $9^k$ is obtained for $T_\AAA(k)$
  by setting $r=3$, and in case of (3) $T_\AAA(k)$ is at most $5^k$.
  It follows that $T_\AAA(k)\leq 9^k$.
\end{proof}

We now turn to the case when the base classes are either $\HORN \cup
\ZVAL$ or $\AHORN \cup \OVAL$.

As before, in scenario (1), the length of an obstructing clause must
be bounded in a \textsc{Yes}-instance. Indeed, consider a clause $C$
that is not in $\HORN \cup \ZVAL$. Such a clause has only positive
literals, and any strong backdoor set must contain all these literals but
one. Therefore, if the clause has more than $k+1$ literals, we
reject the instance, otherwise the clause length is bounded, allowing
for a bounded branching strategy. 

In scenario (2), on the other hand, we have a pair of obstructing
clauses, say $C \in \HORN \setminus \ZVAL$, and $C' \in \ZVAL
\setminus \HORN$. Now, note that $C$ can have only one literal by
definition. Let $O$ be any subset of two positive literals from
$C'$. Note that any strong backdoor must intersect either $O$ or the
unique literal in $C$, which again leads us to a feasible branching
step. We now turn to a formal description of the algorithm.

\begin{THE}\label{the:sat-fpt-horn-zval}
  \STBDET{(\HORN \cup \ZVAL)} and \STBDET{(\AHORN \cup \OVAL)}
  are fixed-parameter tractable parameterized by the size of the backdoor set.
\end{THE}
\begin{proof}
  We only give the proof for the class $\HORN \cup \ZVAL$, 
  as the proof for the class $\AHORN \cup \OVAL$ is analogous.
  Let $\HHH=\HORN \cup \ZVAL$. We will provide an algorithm
  $\AAA$ solving \SATBR{\HHH} in time $O(2^k|F|)$ and
  satisfying $T_\AAA(k)=3^k$. The result then follows from Theorem~\ref{the:sat-det}.

  Let $F$ be the given CNF formula with variables $V$, $k$ the given
  integer, and $B' \subseteq V$ with $|B'|\leq k$. The algorithm
  $\AAA$ first checks whether $F[\tau] \in \HHH$ for every
  assignment $\Bassgn{\tau}{B'}$. If that is the case then
  $\AAA$ returns \textsc{Yes}. Otherwise let 
  $\Bassgn{\tau}{B'}$ be an assignment such that
  $F[\tau] \notin \HHH$, we consider two cases: 
  \begin{enumerate}
  \item if $|B'|=k$, then there is no strong $\HHH$\hy backdoor set
    $B$ of
    $F$ with $B' \subseteq B$ and $|B|\leq k$ and the algorithm $\AAA$
    correctly returns the set $\mathcal{Q}=\emptyset$,
  \item otherwise, i.e., if $|B'|<k$, then
    we again distinguish two cases:
    \begin{enumerate}
    \item there is a clause $C \in F[\tau]$ with $C \notin \HHH$, 
      then we again distinguish two cases:
      \begin{enumerate}
      \item $|B|+|C|-1>k$, then the algorithm $\AAA$ returns the set
        $\mathcal{Q}=\emptyset$.

        \medskip
        To show the correctness of this step it is sufficient to show
        that $F$ does not have a strong $\HHH$\hy backdoor set $B$
        with $B' \subseteq B$ and $|B|\leq k$. Assume for a
        contradiction that this is not the case and let $B$ be
        strong $\HHH$\hy backdoor set of $F$ with $B' \subseteq B$
        and $|B|\leq k$. Because $|B'|+|C|-1>k$ we obtain that $B$
        misses at least two literals of $C$. Let $\Bassgn{\tau'}{B}$ be any assignment 
        that is compatible with $\tau$ and does not satisfy any literal from
        $C$. Then $F[\tau']$ contains a subclause of $C$ that contains at least two
        literals. It follows that $F[\tau'] \notin \HHH$,
        contradicting our assumption that $B$ is a strong
        $\HHH$-backdoor set for $F$. 
      \item otherwise, i.e., if $|B|+|C|-1 \leq k$, then $\AAA$
        returns the set $\mathcal{Q}=\SB \mbox{var}(C) \setminus \{v\} \SM v \in
        \mbox{var}(C) \SE$.

        \medskip
        To show the correctness of this step assume that there is a
        strong $\HHH$\hy backdoor set $B$ of $F$ with $B' \subseteq B$
        and $|B|\leq k$. It suffices to show that $|\mbox{var}(C)
        \setminus B|\leq 1$. Assume not, and let $\Bassgn{\tau^*}{B}$ be an assignment
        compatible with $\tau$ that does not satisfy any literal of $C$. 
        Then, $C[\tau^*] \in F[\tau^*]$ and because $|C[\tau^*]|>1$ and 
        $C[\tau^*] \notin \ZVAL$, also $F[\tau^*] \notin \HORN
        \cup \ZVAL$, contradicting our assumption that $B$ is a strong
        $\HHH$\hy backdoor set of $F$. 
      \end{enumerate}
    \item otherwise, there is a clause $C$ in $F[\tau]$ with $C \in
      \HORN \setminus \ZVAL$ and a clause $C'$ in $F[\tau]$ with $C \in
      \ZVAL \setminus \HORN$. Let $O$ be the set of variables
      containing the one variable from $C$ plus any two variables
      corresponding to two positive literals from $C'$. Then $\AAA$
      returns the set $\mathcal{Q}=\SB \{v\}\SM v \in O\SE$.

      \medskip
      To show the correctness of this step assume that there is a
      strong $\HHH$\hy backdoor set $B$ of $F$ with $B' \subseteq B$
      and $|B|\leq k$. It sufficies to show that $B \cap
      O \neq \emptyset$. Assume not, and let $\Bassgn{\tau^*}{B}$ be
      an assignment compatible with $\tau$ that does not satisfy any
      literal of $C'$. Then, $F[\tau^*]$ contains the clause $C$
      together with a subclause $C''$ of $C'$ containing at least two positive
      literals. Because $C \notin \ZVAL$ and $C'' \notin \HORN$, we
      obtain that $F[\tau^*] \notin \HHH$.
    \end{enumerate}
  \end{enumerate}
  This completes the description and the proof of correctness of the
  algorithm $\AAA$. The running time of $\AAA$ is the
  number of assignments of the at most $k$ variables in $B'$ times the
  time required to test whether the reduced formula is in $\HHH$,
  i.e., $\AAA$ runs in time $O(2^k|F|)$. It remains to obtaine
  the stated bound of $3^k$ for the function $T_\AAA(k)$.
  According to the two branching rules given in 
  (2.a.ii) and (2.b) we obtain that $T_\AAA(k)$ can be
  bounded by the maximum of the following recurrence relations:
  (1) $T_\AAA(k)=rT_\AAA(k-(r-1))$ for any $r \geq 2$ (2.a.ii) and
  (2) $T_\AAA(k)=3T_\AAA(k-1)$ (2.b). In case of (1) the maximum of $2^k$ is
  obtained for $T_\AAA(k)$ by setting $r=2$, and in case of (2) $T_\AAA(k)$ is at most $3^k$.
  It follows that $T_\AAA(k)\leq 3^k$.
\end{proof}

In our next result, we consider heterogeneous base classes comprised
of three Schaefer's classes, namely $\KROM \cup \HORN \cup \ZVAL$ and
$\KROM \cup \AHORN \cup \OVAL$. We refer the reader to
Figure~\ref{fig:branching_triples} for an overview of the branching
strategies employed here. 

\begin{THE}\label{the:sat-fpt-triple}
  \STBDET{(\KROM \cup \HORN \cup \ZVAL)} and \STBDET{(\KROM \cup \AHORN \cup \OVAL)}
  are fixed-parameter tractable parameterized by the size of the backdoor set.
\end{THE}
\begin{proof}
  We only give the proof for the class $\KROM \cup \HORN \cup \ZVAL$, 
  since the proof for the class $\KROM \cup \AHORN \cup \OVAL$ is
  analogous.
  Let $\HHH=\KROM \cup \HORN\cup \ZVAL$. We will provide an algorithm
  $\AAA$ solving \SATBR{\HHH} in time $O(2^k|F|)$ and
  satisfying $T_\AAA(k)=9^k$. The result then follows from Theorem~\ref{the:sat-det}.

  Let $F$ be the given CNF formula with variables $V$, $k$ the given
  integer, and $B' \subseteq V$ with $|B'|\leq k$. The algorithm
  $\AAA$ first checks whether $F[\tau] \in \HHH$ for every
  assignment $\Bassgn{\tau}{B'}$. If that is the case then
  $\AAA$ returns \textsc{Yes}. Otherwise let 
  $\Bassgn{\tau}{B'}$ be an assignment such that
  $F[\tau] \notin \HHH$, we consider two cases: 
  \begin{enumerate}
  \item if $|B'|=k$, then there is no strong $\HHH$\hy backdoor set $B$ of
    $F$ with $B' \subseteq B$ and $|B|\leq k$ and the algorithm $\AAA$
    correctly returns the set $\mathcal{Q}=\emptyset$,
  \item otherwise, i.e., if $|B'|<k$, then
    we again distinguish two cases:
    \begin{enumerate}
    \item there is a clause $C \in F[\tau]$ with $C \notin \HHH$, then
      we again distinguish two cases:
      \begin{enumerate}
      \item $|B'|+|C|-2>k$, then the algorithm $\AAA$ returns the set
        $\mathcal{Q}=\emptyset$.

        \medskip
        To show the correctness of this step it is sufficient to show
        that $F$ does not have a strong $\HHH$\hy backdoor set $B$
        with $B' \subseteq B$ and $|B|\leq k$. Assume for a
        contradiction that this is not the case and let $B$ be
        strong $\HHH$\hy backdoor set of $F$ with $B' \subseteq B$
        and $|B|\leq k$. Because $|B'|+|C|-2>k$, we obtain that $B$
        misses at least three literals from $C$.
        Hence, there is an assignment $\Bassgn{\tau'}{B}$ that is compatible with
        $\tau$ such that $F[\tau']$ contains a clause $C' \subseteq C$ with
        $|C'|>2$. Clearly, $C' \notin \KROM \cup \ZVAL$ and because $C$
        contained only positive literals (because $C \notin \ZVAL)$) also 
        $C' \notin \HORN$. Hence, $C' \notin \KROM \cup \HORN \cup \ZVAL$ 
        contradicting our assumption that $B$ is a strong
        $\HHH$-backdoor set for $F$.
      \item otherwise, i.e., if $|B|+|C|-2 \leq k$, then the algorithm $\AAA$
        returns the set $\mathcal{Q}=\SB \mbox{var}(C \setminus O) \SM O
        \subseteq \mbox{var}(C) \land |O|=2 \SE$.

        \medskip
        To show the correctness of this step assume that there is a
        strong $\HHH$\hy backdoor set $B$ of $F$ with $B' \subseteq B$
        and $|B|\leq k$. It suffices to show that $|\mbox{var}(C) \setminus B|\leq
        2$. Assume not, and let $\Bassgn{\tau^*}{B}$ be an assignment
        compatible with $\tau$ that does not satisfy any literal of $C$. 
        Then, $C[\tau^*] \in F[\tau^*]$ and because $|C[\tau^*]|>2$ and 
        $C[\tau^*]$ contains only positive literals, also
        $F[\tau^*] \notin \KROM \cup \HORN \cup \ZVAL$, contradicting
        our assumption that $B$ is a strong $\HHH$-backdoor set for $F$.
      \end{enumerate}
    \item there is a clause $C \in F[\tau]$ with $C \in \KROM$ and $C
      \notin \ZVAL$, we again distinguish two cases:
      \begin{enumerate}
      \item there is a clause $C' \in F[\tau]$ with $C' \notin \KROM
        \cup \HORN$ (note that $C' \in \ZVAL$). Let $O$ be a set
        of variables containing all variables of $C$ (at most two) plus three
        variables of $C'$ of which two correspond to positive literals
        in $C'$. Then the algorithm $\AAA$ returns the set $\mathcal{Q}=\SB
        \{v\} \SM v \in O\SE$.

        \medskip
        To show the correctness of this step assume that there is a
        strong $\HHH$\hy backdoor set $B$ of $F$ with $B' \subseteq B$
        and $|B|\leq k$. It sufficies to show that $B \cap
        O \neq \emptyset$. Assume not and let $\Bassgn{\tau^*}{B}$ be an assignment
        compatible with $\tau$, which does not satisfy any literal of
        $C'$. Then, $F[\tau^*]$ contains $C$ as well as a subclause $C''$ of
        $C'$ of length at least three, which contains at least two
        positive literals. Then, $C'' \notin \HORN \cup KROM$ and because
        $C \notin \ZVAL$, we obtain that $F[\tau^*] \notin \KROM \cup
        \HORN \cup \ZVAL$, contradicting
        our assumption that $B$ is a strong $\HHH$-backdoor set for $F$.
      \item otherwise, there are clauses $C_\KROM, C_\HORN \in
        F[\tau]$ with $C_\KROM \in \KROM \setminus \HORN$ and $C_\HORN
        \in \HORN \setminus \KROM$. Let $O$ be a set of variables
        containing all variables of $C_\KROM$ (at most two)
        and three variables of $C_\HORN$. Then the algorithm $\AAA$
        returns the set $\mathcal{Q}=\SB \{v\}\SE v \in O\SE$.

        \medskip
        To show the correctness of this step assume that there is a
        strong $\HHH$\hy backdoor set $B$ of $F$ with $B' \subseteq B$
        and $|B|\leq k$. It sufficies to show that $B \cap
        O \neq \emptyset$. Assume not and let $\Bassgn{\tau^*}{B}$ be
        an assignment compatible with $\tau$, which does not satisfy any literal of
        $C_\HORN$. Then, $F[\tau^*]$ contains $C_\KROM$ as well as a subclause $C'$ of
        $C_\HORN$ of length at least three. Because $C_\KROM \in \KROM
        \setminus \HORN$ it follows that $C_\KROM \notin \ZVAL$. Because
        also $C' \notin \KROM$,
        we obtain that $F[\tau^*] \notin \KROM \cup \HORN \cup \ZVAL$, 
        contradicting
        our assumption that $B$ is a strong $\HHH$-backdoor set for $F$.
      \end{enumerate}
    \item there is a clause $C \in F[\tau]$ with $C \in \KROM$ and $C
      \notin \HORN$. In this case also $C \notin \ZVAL$ and hence the
      case is covered by case (2.b).
    \item The cases (2.a), (2.b), and (2.c) completely cover all
      possible cases, because if none of these cases apply then
      $F[\tau] \in \ZVAL$, which contradicts the choice of $\tau$. To
      see this assume there is a clause $C \in F[\tau]$ with $C \notin
      \ZVAL$. Because of the cases (2.b) and (2.c), we obtain that $C \notin
      \KROM$. Furthermore, because of case (2.a), we obtain that $C
      \in \HORN$. Hence, $C \in \HORN \setminus (\ZVAL \cup \KROM)$ but
      such a clause cannot exist.
    \end{enumerate}
  \end{enumerate}
  This completes the description and the proof of correctness of the
  algorithm $\AAA$. The running time of $\AAA$ is the
  number of assignments of the at most $k$ variables in $B'$ times the
  time required to test whether the reduced formula is in $\HHH$,
  i.e., $\AAA$ runs in time $O(2^k|F|)$. It remains to obtain
  the stated bound of $9^k$ for the function $T_\AAA(k)$.

  According to the three branching rules given in 
  (2.a.ii), (2.b.i), and (2.b.ii) we obtain that $T_\AAA(k)$ can be
  bounded by the maximum of the following recurrence relations:
  (1) $T_\AAA(k)=r^2T_\AAA(k-(r-2))$ for any $r>2$ (2.a.ii), 
  (2) $T_\AAA(k)=5T_\AAA(k-1)$ (2.b.i), 
  and (3) $T_\AAA(k)=5T_\AAA(k-1)$ (2.b.ii). In case of (2) and (3) $T_\AAA(k)$ is at most 
  $5^k$ and in case of (1) the maximum of $9^k$ is obtained for $T_\AAA(k)$
  by setting $r=3$.
  It follows that $T_\AAA(k)\leq 9^k$.

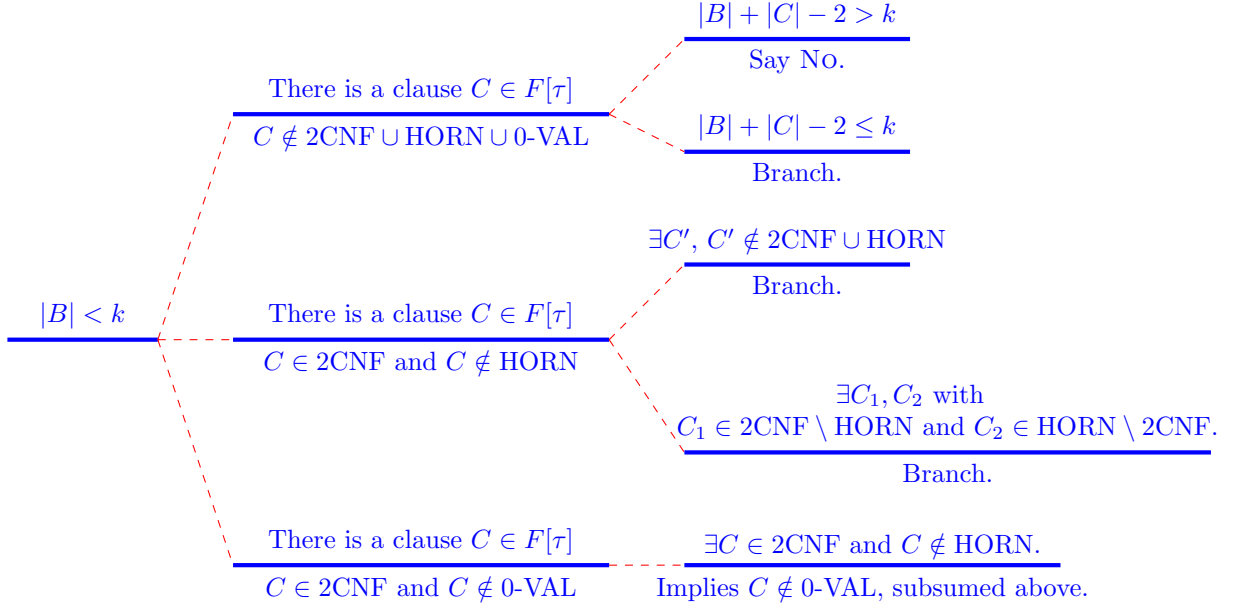
\begin{figure}
\begin{tikzpicture}
  \draw[level] (0,0) -- node[above] {$|B| < k$} (2,0);

  \draw[connect] (2,0)  -- (3,3) (2,0) -- (3,-3) (2,0) -- (3,0);
  \draw[level]   (3,3)  -- node[above] {There is a clause $C \in F[\tau]$} node[below] {$C \notin \KROM \cup
      \HORN \cup \ZVAL$} (8,3);
  \draw[level]   (3,0)  -- node[above] {There is a clause $C \in F[\tau]$} node[below] {$C \in \KROM$ and $C \notin \HORN$} (8,0);
  \draw[level]   (3,-3)  -- node[above] {There is a clause $C \in F[\tau]$} node[below] {$C \in \KROM$ and $C \notin \ZVAL$} (8,-3);
  
  \draw[level]   (9,4)  -- node[above] {$|B| + |C| - 2 > k$} node[below] {Say \textsc{No}.} (12,4);
  \draw[level]   (9,2.5)  -- node[above] {$|B| + |C| - 2 \leq k$} node[below] {Branch.} (12,2.5);

  \draw[level]   (9,1)  -- node[above] {$\exists C'$, $C' \notin \KROM \cup \HORN$} node[below] {Branch.} (12,1);
  \draw[level]   (9,-1.5)  -- node[above] {$C_1 \in \KROM \setminus \HORN$ and $C_2        \in \HORN \setminus \KROM$.} node[below] {Branch.} (16,-1.5);

  \node[level] at (12,-0.75) {$\exists C_1, C_2$ with};

  \draw[level]   (9,-3)  -- node[above] {$\exists C \in \KROM$ and $C \notin \HORN$.} node[below] {Implies $C \notin \ZVAL$, subsumed above.} (14,-3);

  \draw[connect] (8,3)  -- (9,4) (8,3) -- (9,2.5) (8,0) -- (9,1) (8,0) -- (9,-1.5) (8,-3) -- (9,-3);

\end{tikzpicture}
\caption{The branching cases for Theorem~\ref{the:sat-fpt-triple}.}
\label{fig:branching_triples}
\end{figure}

\end{proof}

We say that a pair $(\CCC,\CCC')$ of Schaefer classes 
is a \emph{bad pair} if $\CCC \in \{\HORN,\ZVAL\}$ and $\CCC' \in
\{\AHORN,\OVAL\}$. Our next result establishes hardness for the case
when the base class contains a bad pair. 
\begin{THE}\label{the:sat-bp-w2}
  For every $S \subseteq \SCHAE$ that contains at least one
  bad pair of Schaefer classes, it holds that
  \STBDET{\CCC} is $\W[2]$\hy hard parameterized by the size of the backdoor set, where
  $\CCC=\bigcup_{s \in S}s$.
\end{THE}
\begin{proof}
  We give a parameterized reduction from
  the $\W[2]$\hy complete {\sc Hitting Set} problem.
  Given an instance $({\mathcal Q}, U,k)$ for {\sc Hitting
    Set}, construct a formula $F$ as follows. The variables of $F$ are
  ${\mathcal Q} \cup \SB d_Q^1,\dotsc,d_Q^{|U\setminus Q|} \SM Q \in
  \mathcal{Q}\SE$. 
  For each set $Q\in \mathcal{Q}$, there is one
  clause $c_Q = Q \cup \SB d_Q^1, \dotsc, d_Q^{|U\setminus Q|}\SE$. There is also one clause
  $c_{U} = \{\neg u : u\in U\}$. 
  This completes the description of the reduction.

  We claim that $\mathcal{Q}$ has a hitting set of size at most $k$ if
  and only if the formula $F$ has a strong $\CCC$-backdoor set of
  size at most $k$.
  Suppose $X \subseteq U$, $|X|\le k$, is a hitting set.
  To show that $X$ is also a strong $\CCC$-backdoor set, consider
  any assignment $\Bassgn{\tau}{X}$. If $\tau(x)=0$ for some $x\in X$, then
  $\tau$ satisfies the clause $c_{U}$. Thus, $F[\tau] \in
  \CCC$ since each clause in $F[\tau]$ contains no negative
  literal and at least one positive literal and is hence in $\OVAL \cap
  \AHORN$.
  If $\tau(x)=1$ for all $x\in X$, then all clauses $c_Q, Q\in
  \mathcal{Q}$, are satisfied by $\tau$ since $X$ is a hitting set. The
  only remaining clause is $\HORN$ and $\ZVAL$ since it has no positive
  literal and at least one negative literal.
  
  For the other direction, suppose that $X$ is a strong
  $\CCC$-backdoor set of size at most $k$. Obtain $X'$ from
  $X$ by replacing each $d_Q^i \in X$ for any $i$ with $1 \leq i \leq
  |U\setminus Q|$ by some variable from $Q$.
  
  The set $X'$ is also a strong $\CCC$-backdoor set of size at
  most $k$. 
  Therefore, the assignment $\Bassgn{\tau}{X'}$ with $\tau(x)=1$ for all
  $x\in X'$ must satisfy all clauses $c_Q, Q\in \mathcal{Q}$. Otherwise,
  $F[\tau]$ contains both a long clause containing only positive and a
  long clause containing only negative literals and hence $F[\tau]
  \notin \bigcup_{s \in S} s$.
  Thus, $X'$ is a hitting set for $\mathcal{Q}$ of size at most $k$.
\end{proof}

We now summarize our results in the following theorem. 

\begin{THE}\label{the:sat-dico}
  Let $S \subseteq \SCHAE$ and $\CCC=\bigcup_{s \in S}s$. 
  Then, \STBDET{\CCC} is fixed parameter tractable parameterized by the size of the backdoor set
  if $S$ does not contain a bad pair of Schaefer classes,
  otherwise it is $\W[2]$-hard.
\end{THE}
\begin{proof}
  If $S$ contains a bad pair then the result follows from
  Theorem~\ref{the:sat-bp-w2}. So assume that $S$ does not contain a
  bad pair, we distinguish the following cases:
   \begin{itemize}
   \item if $|S|=1$, then the theorem follows from known results, see,
     for example,~\cite{NishimuraRagdeSzeider04-informal,GaspersSzeider12}.
   \item if $|S|=2$, then either $\KROM \in S$ in which case the
     result follows from Theorem~\ref{the:sat-fpt-krom}, or
     $S=\{\HORN,\ZVAL\}$ or $S=\{\AHORN,\OVAL\}$ in which case the
     result follows from Theorem~\ref{the:sat-fpt-horn-zval}.
   \item if $|S|=3$, then either $\KROM \in S$ and hence 
     $S=\{\KROM,\HORN,\ZVAL\}$ or $S=\{\KROM,\AHORN,\OVAL\}$ in which
     case the result follows from Theorem~\ref{the:sat-fpt-triple}, or
     $\KROM \notin S$ in which case $S$ contains at least one bad pair,
     and the result follows from Theorem~\ref{the:sat-bp-w2}.
   \item if $|S|>3$, then $S$ contains at least one bad pair, and the 
     result follows from Theorem~\ref{the:sat-bp-w2}.
   \end{itemize}
\end{proof}

It is crucial for these hardness proofs that clauses have unbounded length.
Indeed, if clause-lengths are bounded or if we add the maximum clause
length to the parameter, then strong backdoor detection becomes FPT
for any combination of Schaefer classes.

\begin{THE}\label{the:sat-clause-sb}
  Let $\CCC$ be a class of CNF-formulas such that $\CCC=\bigcup_{s \in S}s$ 
  for some $S \subseteq \SCHAE$. Then, \STBDET{\CCC}  
  is fixed-parameter tractable 
  for the combined parameter $k$ and the
  maximum length $r$ of a clause of the input formula.
\end{THE}
\begin{proof}
  We will provide an algorithm
  $\AAA$ solving \SATBR{\CCC} in time $O(2^k|F|)$ and
  satisfying $T_\AAA(k)={(|S|r))}^{k+1}$. The result then follows from
  Theorem~\ref{the:sat-det}.

  Let $F$ be the given CNF formula with variables $V$, $k$ the given
  integer, and $B' \subseteq V$ with $|B'|\leq k$. The algorithm
  $\AAA$ first checks whether $F[\tau] \in \CCC$ for every
  assignment $\Bassgn{\tau}{B'}$. If that is the case then
  $\AAA$ returns \textsc{Yes}. Otherwise let 
  $\Bassgn{\tau}{B'}$ be an assignment such that
  $F[\tau] \notin \CCC$, we consider two cases:
  \begin{enumerate}
  \item if $|B'|=k$, then there is no strong $\CCC$\hy backdoor set $B$
    of $F$ with $B'\subseteq B$ and $|B|\leq k$ and the algorithm $\AAA$
    correctly returns the set $\mathcal{Q}=\emptyset$.
  \item otherwise, i.e., if $|B'|<k$, then let $C_s$ be an arbitrary
    clause in $F[\tau]$ with $C_s\notin s$ for every $s \in S$.
    The algorithm $\AAA$ then returns the set $\mathcal{Q}=\bigcup_{s\in S}\SB
    \{v\} \SM v \in \mbox{var}(C_s)\SE$.
    
    \medskip
    To show the correctness of this step assume that there is a
    strong $\CCC$\hy backdoor set $B$ of $F$ with $B' \subseteq B$
    and $|B|\leq k$. It sufficies to show that $B \cap \bigcup_{s \in
      S}\mbox{var}(C) \neq \emptyset$.
    Let $\Bassgn{\tau^*}{B}$ be any
    assignment of the variables in $B$ that agrees with $\tau$ on the
    variables in $B'$ and let $s \in S$ be such that $F[\tau^*] \in s$. 
    Because $B$ is a strong $\CCC$-backdoor set,
    $s$ clearly exists. We claim that $\mbox{var}(C) \cap B\neq
    \emptyset$ for every clause $C \in F[\tau]$ with $C \notin S$ (and
    hence in particular for the clause $C_s$ chosen by $\AAA$).
    Indeed,
    suppose not. Then $C \in F[\tau^*]$ a contradiction to our
    assumption that $F[\tau^*] \in s$.
  \end{enumerate}
  This completes the description and the proof of correctness of the
  algorithm $\AAA$. The running time of $\AAA$ is the
  number of assignments of the at most $k$ variables in $B'$ times the
  time required to test whether the reduced formula is in $\CCC$,
  i.e., $\AAA$ runs in time $O(2^k|F|)$. It remains to obtain
  the stated bound of ${(|S|r)}^{k+1}$ for the function $T_\AAA(k)$,
  which follows because $|\mathcal{Q}|\leq |S|r$ and $|Q|=1$ for every $Q \in \mathcal{Q}$.
\end{proof}

\subsection{Weak Backdoor Sets}

It turns out that the complexity of finding weak backdoors into
heterogeneous classes is tightly tied to the complexity of finding
weak backdoors into the individual base classes. Let us briefly
consider the intuition for this scenario before stating the
result. Suppose we are given algorithms for finding weak backdoors
into base classes $\CCC_1$ and $\CCC_2$. Observe that a CNF-formula
has a weak backdoor of size at most $k$ into $(\CCC_1 \cup \CCC_2)$
if, and only if, it has a weak backdoor of size at most $k$ into at
least one of $\CCC_1$ or $\CCC_2$. Thus, if we have a positive
algorithmic result for the weak backdoor question with respect to the
individual base classes, then these algorithms can be easily combined
to obtain an algorithm for finding weak backdoors into the
corresponding heterogeneous base class. On the other hand, a hardness
result for even one of the involved classes can usually be leveraged
to obtain a hardness result for the entire heterogeneous problem, by
adding gadgets that isolate the issue to only one of the base classes,
typically the one that was hard to begin with. In particular, we obtain
the following result.

\begin{THE}
\label{wbd:w-hard}
  Let $S \subseteq \SCHAE$ and $\CCC=\bigcup_{s \in S}s$. 
  Then, \WEBDET{\CCC} is $\W[2]$-hard parameterized by the size of the backdoor set.
\end{THE}
Before we proceed further, we introduce the notion of an obstruction
which we will use in the proof.  Specifically, for $S, S' \in
\SCHAE$, we say that a clause $C$ is an $S$\hy obstruction to $S'$ if
$C \in S$ and $C \notin S'$. For any pair of Schaefer classes $S$
and $S'$,  it is easy to construct examples of $S$\hy obstructions to
$S'$. In particular, we have:
\begin{enumerate}
\item $(\neg x, \neg y, \neg z)$ is a $\HORN$\hy obstruction to $\KROM$, and also a $\ZVAL$\hy obstruction to $\KROM$.
\item $(x,y,z)$ is an $\AHORN$\hy obstruction to $\KROM$, and also a $\OVAL$\hy obstruction to $\KROM$.
  
\item $(x, y, \neg z)$ is a $\ZVAL$\hy obstruction to $\HORN$.
\item $(x,y)$ is an $X$\hy obstruction to $\HORN$ for all $X \in
  \{\KROM$, $\AHORN$, $\OVAL\}$.
  
\item $(\neg x, \neg y, z)$ is a $\OVAL$\hy obstruction to $\AHORN$.
\item $(\neg x, \neg y)$ is an $X$\hy obstruction to $\AHORN$ for all $X \in \{\KROM$, $\HORN$, $\ZVAL\}$.
  
\item $(x)$ is an $X$\hy obstruction to $\OVAL$ for all $X \in \{ \KROM$, $\OVAL$, $\HORN$, $\AHORN\}$.
\item $(\neg x)$ is an $X$\hy obstruction to $\OVAL$ for all $X \in \{\KROM$, $\OVAL$, $\HORN$, $\AHORN\}$.
\end{enumerate}

  We are now ready to describe our reduction towards Theorem~\ref{wbd:w-hard}.

\begin{proof} If $|S| = 1$, then the result follows directly from~\cite[Proposition 1]{GaspersSzeider12}. Otherwise, let $S \in \CCC$. We reduce an instance $(F,k)$ of \WEBDET{S}  to \WEBDET{\CCC}.  To this end, for all $S' \in \CCC$ such that $S \neq S'$, introduce $(k+1)$ variable-disjoint copies of a $S$-obstruction to $S'$ using fresh variables for each obstruction, and in particular, disjoint from $\mbox{var}(F)$. We use $F'$ to denote the formula obtained from $F$ after adding these obstructions to $F$, and the reduced instance is  given by $(F',k)$. 

In the forward direction, if $X$ is a weak backdoor to $S$ for $F$, and if $\tau: X \rightarrow \{0,1\}$ is such that $F[\tau] \in S$, then observe that $F'[\tau] \in S$, since all the added obstructions (which are unaffected by $\tau$) were in $S$ by definition. It is also easy to see that any satisfying assignment for $\mbox{var}(F) \setminus X$ can be easily extended to an assignment involving the new variables that satisfy all the obstructions (recall that the obstructions were variable disjoint, and therefore can be easily satisfied independently). 

On the other hand,  let $X$ be a weak backdoor to $\CCC$ for $F'$. Since $X$ has at most $k$ variables, for any $\tau: X \rightarrow \{0,1\}$, $F'[\tau]$ contains at least one $S'$-obstructing clause for each $S' \in \CCC$ such that $S' \neq S$. Therefore, $F'[\tau]$ must necessarily belong to $S$, and therefore $X \cap \mbox{var}(F)$ is easily checked to be a weak backdoor to $S$ for $F$, as desired. 
\end{proof}

On the other hand, as with the strong backdoors, we do obtain tractability when we have bounded clause lengths. We establish this formally below. 

\begin{THE}\label{the:sat-clause-wb}
  Let $\CCC$ be a class of CNF-formulas such that $\CCC=\bigcup_{s \in S}s$ 
  for some $S \subseteq \SCHAE$. Then, \WEBDET{\CCC} 
  is fixed-parameter tractable for the combined parameter $k$ and the
  maximum length $r$ of a clause of the input formula.
\end{THE}
\begin{proof}
  The algorithm uses a depth-bounded
  search tree approach to find a weak $\CCC$-backdoor set
  of size at most $k$. Let $F$ be any CNF-formula with variables $V$
  and let $r$ be the maximum length of any clause of $F$.

  We construct a search tree $T$, for which
  every node is labeled by a set $B$ of at most $k$ variables of
  $V$. Additionally, every leaf node has a second label,
  which is either \textsc{Yes} or \textsc{No}. 
  $T$ is defined inductively as follows. The root of $T$ is labeled by
  the empty set. Furthermore, if $t$ is a node of $T$,
  whose first label is $B$, the
  children of $t$ in $T$ are obtained as follows.

  If there is an assignment $\Bassgn{\tau}{B}$, such that $F[\tau] \in
  \CCC$ and $F[\tau]$ is satisfiable, then $B$ is a
  weak $\CCC$-backdoor set of size at most $k$, and hence $t$ becomes a
  leaf node, whose second label is \textsc{Yes}. Otherwise, i.e., 
  for every assignment $\Bassgn{\tau}{B}$ either $F[\tau] \notin \CCC$ or
  $F[\tau] \in \CCC$ but $F[\tau]$ is not satisfiable, we
  consider two cases: (1) $|B|=k$, then $t$ becomes a leaf node, whose
  second label is \textsc{No}, and (2) $|B|<k$,
  then for every assignment $\Bassgn{\tau}{B}$ such that $F[\tau] \notin
  \CCC$, for every $s \in S$, and every variable $v$ 
  that occurs in some clause $C \notin s$, 
  $t$ has a child whose first label is $B \cup \{v\}$.

  If $T$ has a leaf node, whose second label is \textsc{Yes}, then the
  algorithm returns the first label of that leaf node. Otherwise the
  algorithm returns \textsc{No}. This completes the description of the
  algorithm.

  We now show the correctness of the algorithm. First, suppose 
  the search tree $T$ built by the algorithm has a leaf node $t$ whose
  second label is \textsc{Yes}. Here, the algorithm returns the first label, say
  $B$ of $t$. By definition of $T$, we obtain that $|B|\leq k$ and 
  the set $B$ is a weak $\CCC$\hy backdoor set of $F$, as required.

  Now consider the case where the algorithm returns
  \textsc{No}. We need to show that there is no weak
  $\CCC$\hy backdoor set of size at most $k$ for $F$. 
  Assume, for the sake of contradiction that such a set $B$ exists.

  Observe that if $T$ has a leaf node
  $t$ whose first label is a set $B'$ with $B' \subseteq B$,
  then the second label of $t$ must be \textsc{Yes}. This is because,
  either $|B'|<k$ in which case the second label of $t$ must be
  \textsc{Yes}, or $|B'|=k$ in which case $B'=B$ and by the definition
  of $B$ it follows that the second label of $t$ must be \textsc{Yes}.

  It hence remains to show that $T$ has a leaf node whose first label is a set $B'$
  with $B' \subseteq B$. This will complete the proof about the
  correctness of the algorithm. We will show a slightly stronger
  statement, namely, that for every natural number $\ell$, either $T$
  has a leaf whose first label is contained in
  $B$ or $T$ has an inner node of distance exactly $\ell$ from the root
  whose first label is contained in $B$. We show the latter by
  induction on $\ell$.

  The claim obviously holds for $\ell=0$. So assume that $T$ contains a
  node $t$ at distance $\ell$ from the root of $T$ whose first label, say
  $B'$, is a subset of $B$. 
  If $t$ is a leaf node of $T$, then the
  claim is shown. Because $B$ is a weak $\CCC$\hy backdoor set for
  $F$, there is an assignment $\Bassgn{\tau^*}{B}$ such that $F[\tau^*] \in
  \CCC$ and $F[\tau^*]$ is satisfiable. In particular, there is some
  $s \in S$ such that $F[\tau^*] \in s$. Let $\Bassgn{\tau}{B'}$ be the
  assignment that agrees with $\tau^*$ on the variables of
  $B'$. Because $t$ is not a leaf node, there is a clause $C \in
  F[\tau]$ with $C \notin s$ such that $t$ has a child $t'$ for
  every $v \in \mbox{var}(C)$.
  We claim that $\mbox{var}(C) \cap B \neq \emptyset$ and hence $t$ has
  a child, whose first label is a subset of $B$, as required. Indeed,
  suppose not. Then $C \in F[\tau^*]$ a contradiction to our
  assumption that $F[\tau^*] \in s$.
  This concludes our proof concerning the correctness of the algorithm.

  The running time of the algorithm is obtained as
  follows. Let $T$ be a search tree obtained by the
  algorithm. Then the running time of the depth-bounded search tree
  algorithm is $O(|V(T)|)$ times the maximum time that is spent on any
  node of $T$. Since the number of children of any node of $T$ is
  bounded by $2^k|S|r$ (recall that $r$ denotes the maximum 
  length of any clause of $F$) and the
  longest path from the root of $T$ to some
  leaf of $T$ is bounded by $k+1$, we obtain that $|V(T)| \leq
  O({(2^k|S|r)}^{k+1})$. Furthermore, the time
  required for any node $t$ of $T$
  is at most $O(2^k|F|)$.
  Putting everything together, we obtain
  $O((2^k|S|r)^{k+1}2^{k}|F|)$,
  as the total running time of the algorithm.
  This
  shows that \WEBDET{\CCC} is fixed-parameter tractable
  parameterized by $k$, $r$.
\end{proof}

\smallskip
We close this section by noting that backdoor sets with \emph{empty
  clause detection}, as proposed by \cite{DilkinaGomesSabharwal07} can
be considered as backdoor sets into the heterogeneous base class
obtained by the union of a homogeneous base class $\CCC$ and the class
of all formulas that contain the empty clause. The detection of strong
backdoor sets with empty clause detection is not fixed-parameter
tractable for many natural base classes, including Horn and 2CNF
\cite{Szeider08c}.

\section{Base Classes via Closure Properties}

In this section we provide a very general framework that will give
rise to a wide range of heterogeneous base classes~for~CSP.

\sloppypar
Given an $r$-ary relation $R$ over some domain $D$ and an 
$n$-ary operation $\phi:D^n \rightarrow D$, we say that $R$ is {\em closed under
  $\phi$}, if for all collections of $n$ tuples $t_1,\dotsc,t_n$ from
$R$, the tuple $\tuple{\phi(t_1[1],\dotsc,t_n[1]),
  \dotsc,\phi(t_1[r],\dotsc,t_n[r])}$ belongs to $R$. The operation 
$\phi$ is also said to be a {\em polymorphism of $R$}. 
We denote by $\Fun(R)$ the set of all operations $\phi$ such that
$R$ is closed under $\phi$.

Let $\cspi=\tuple{V,D,C}$ be a CSP instance and $c \in C$.  We write
$\Fun(c)$ for the set $\Fun(R(c))$ and we write $\Fun(\cspi)$ for the
set $\bigcap_{c \in C}\Fun(c)$.  We say that $\cspi$ is closed under an
operation $\phi$, or $\phi$ is a polymorphism of $\cspi$, if $\phi \in \Fun(\cspi)$.

We say an operation $\phi$ is \emph{tractable} if every CSP instance
closed under $\phi$ can be solved in polynomial time.

Let $\prop{P}(\phi)$ be a predicate for the operation $\phi$.  We call
$\prop{P}(\phi)$ a \emph{tractable polymorphism property} if the following
conditions hold.
\begin{itemize}
\item There is a constant $c_{\prop{P}}$ such that for all finite
  domains $D$, all operations $\phi$ over $D$ with property
  $\prop{P}$ are of arity at most $c_{\prop{P}}$.
\item Given an operation $\phi$ and a domain $D$, one can check in polynomial time
  whether $\prop{P}(\phi)$ holds on all of the at most $D^{c_{\prop{P}}}$ tuples over $D$,
\item Every operation with property $\prop{P}$ is tractable.
\end{itemize}

Every tractable polymorphism property $\prop{P}$ gives rise to a natural
base class $\CCC_{\prop{P}}$ consisting of all CSP-instances $\cspi$ such that
$\Fun(\cspi)$ contains some polymorphism $\phi$ with $\prop{P}(\phi)$.

In the following we will provide several illustrative examples for
classes of CSP instances that can be defined via tractable
polymorphism predicates. The chosen examples are mostly based on the
Schaefer classes or important generalizations thereof. 
For the definition of the examples we need to define the following
types of operations.
\begin{itemize}
\item An operation $\phi : D^n \rightarrow D$ is \emph{idempotent}
  if for every $d \in D$ it holds that $\phi(d,\dotsc,d)=d$;
\item An operation $\phi : D \rightarrow D$ is \emph{constant} if
  there is a $d \in D$ such that for every $d' \in D$, it
  holds that $\phi(d')=d$;
\item An operation $\phi : D^2 \rightarrow D$ is a
  \emph{min}/\emph{max} operation if there is an ordering of the
  elements of $D$ such that for every $d,d' \in D$, it holds that
  $\phi(d,d')=\phi(d',d)=\min\{d,d'\}$ or
  $\phi(d,d')=\phi(d',d)=\max\{d,d'\}$, respectively;
\item An operation $\phi : D^3 \rightarrow D$ is a \emph{majority}
  operation if for every $d,d' \in D$ it holds that
  $\phi(d,d,d')=\phi(d,d',d)=\phi(d',d,d)=d$;
\item An operation $\phi : D^3 \rightarrow D$ is an \emph{minority}
  operation if for every $d,d' \in D$ it holds
  that $\phi(d,d,d')=\phi(d,d',d)=\phi(d',d,d)=d'$;
\item An operation $\phi : D^3 \rightarrow D$ is a \emph{Mal'cev}
  operation
  if for every $d,d' \in D$ it holds that $\phi(d,d,d')=\phi(d',d,d)=d'$.
\end{itemize}

It is known
that every constant, min/max, majority, minority, and Mal'cev
operation is tractable~\cite{JeavonsCohenGyssens97,BulatovDalmau06}. 
We denote by $\CON$, $\MINMAX$, $\MAJ$, $\AFF$, and $\MAL$ the
class of CSP instances $\cspi$ for which $\Fun(\cspi)$ contains a
constant, a min/max, a majority, a minority, or a Mal'cev
polymorphism, respectively.

Thus $\CON$, $\MINMAX$, $\MAJ$, $\AFF$, and $\MAL$ are the
classes $\CCC_{\prop{P}}$ for $\prop{P} \in \{$constant,  min/max,  majority,  minority,  Mal'cev$\}$, respectively.

In terms of the above definitions we can state the results of
\cite{JeavonsCohenGyssens97} and \cite{BulatovDalmau06} as follows.
\begin{PRO}\label{pro:schaefer-nice}
  Constant, min/max, majority, minority, and Mal'cev are tractable
  polymorphism properties.
\end{PRO}
We want to note here that Proposition~\ref{pro:schaefer-nice} also applies for
other much more general types of operations such as semilattice
operations (sometimes called ACI operation)~\cite{JeavonsCohenGyssens97} (a
generalization of min/max),
$k$-ary near unanimity
operations~\cite{JeavonsCohenCooper98,FederVardi98} (a generalization
of majority), $k$-ary edge
operations~\cite{IdziakMarkovicMcKenzieValerioteWillard10} (a
generalization of Mal'cev), and the two operations of arities three and
four~\cite{KozikKrokhinValerioteWillard15} that capture the bounded width property~\cite{BartoKozik14} (a
generalization of semilattice and near unanimity operations).

In the next sections we will study the problems 
\STBDET{\CCC_{\prop{P}}} for tractable polymorphism properties  $\prop{P}$.

\section{Tractability of Backdoor Detection for CSP}

In this section we will show that both \STBDET{\CCC_{\prop{P}}} and \WEBDET{\CCC_{\prop{P}}}
parameterized by the size of the backdoor
set, the size of the domain, and the maximum arity of the CSP instance
are fixed-parameter tractable for any tractable polymorphism property~$\prop{P}$.
We start by giving the tractability results for strong backdoor sets.
\begin{THE}\label{the:fpt-csp}
  Let $\prop{P}$ be a tractable polymorphism property. Then
  \STBDET{\CCC_{\prop{P}}} is fixed-parameter tractable for the
  combined parameter size of the backdoor set, size of the domain, and
  the maximum arity of the given CSP instance.
\end{THE}
\begin{proof}
  Let $\prop{P}$ be a tractable polymorphism property,  and let
  $\tuple{\cspi,k}$ with $\cspi=\tuple{V,D,C}$ be an instance of
  \STBDET{\CCC_{\prop{P}}}. Let $P$ be the set of all operations on
  $D$ that have property $\prop{P}$. Then, $P$ can be constructed in
  fpt-time with respect to the size of the domain
  , because there are at most
  $|D|^{|D|^{c_{\prop{P}}}}$ $c_{\prop{P}}$-ary operations on $D$
  and for each of them we can test in polynomial time, $|D|^{O(c_{\prop{P}})}$, whether it
  satisfies property $\prop{P}$.  The algorithm uses a depth-bounded
  search tree approach to find a strong $\CCC_{\prop{P}}$-backdoor set
  of size at most $k$.

  We construct a search tree $T$, for which
  every node is labeled by a set $B$ of at most $k$ variables of
  $V$. Additionally, every leaf node has a second label,
  which is either \textsc{Yes} or \textsc{No}. 
  $T$ is defined inductively as follows. The root of $T$ is labeled by
  the empty set. Furthermore, if $t$ is a node of $T$,
  whose first label is $B$, then the
  children of $t$ in $T$ are obtained as follows. If for every
  assignment $\assgn{\tau}{B}{D}$ there is an operation $\phi
  \in P$ such that $\cspi[\tau]$ is closed under $\phi$, then $B$ is a
  strong $P$-backdoor set of size at most $k$, and hence $t$ becomes a
  leaf node, whose second label is \textsc{Yes}. Otherwise, i.e., if
  there is an assignment $\assgn{\tau}{B}{D}$ such that
  $\cspi[\tau]$ is not closed under any operation $\phi \in P$, we
  consider two cases: (1) $|B|=k$, then $t$ becomes a leaf node, whose
  second label is \textsc{No}, and (2) $|B|<k$,
  then for every operation $\phi \in P$ and every variable $v$ in
  the scope of some constraint $c \in C[\tau]$ that is not closed
  under $\phi$, $t$ has a child whose first label is $B \cup \{v\}$.

  If $T$ has a leaf node, whose second label is \textsc{Yes}, then the
  algorithm returns the first label of that leaf node. Otherwise the
  algorithm returns \textsc{No}. This completes the description of the
  algorithm.

  We now show the correctness of the algorithm. First, suppose 
  the search tree $T$ built by the algorithm has a leaf node $t$ whose
  second label is \textsc{Yes}. Here, the algorithm returns the first label, say
  $B$ of $t$. By definition, we obtain that $|B|\leq k$ and for every
  assignment $\assgn{\tau}{B}{D}$, it holds that $\cspi[\tau]$ is
  closed under some operation in $P$, as required.

  Now consider the case where the algorithm returns
  \textsc{No}. We need to show that there is no set $B$ of at most
  $k$ variables of $\cspi$ such that $\Fun(\cspi[\tau]) \cap P \neq \emptyset$
  for every assignment $\tau$ of the variables of $B$. 
  Assume, for the sake of contradiction that such a set $B$ exists.

  Observe that if $T$ has a leaf node
  $t$ whose first label is a set $B'$ with $B' \subseteq B$,
  then the second label of $t$ must be \textsc{Yes}. This is because,
  either $|B'|<k$ in which case the second label of $t$ must be
  \textsc{Yes}, or $|B'|=k$ in which case $B'=B$ and by the definition
  of $B$ it follows that the second label of $t$ must be \textsc{Yes}.

  It hence remains to show that $T$ has a leaf node whose first label is a set $B'$
  with $B' \subseteq B$. This will complete the proof about the
  correctness of the algorithm. We will show a slightly stronger
  statement, namely, that for every natural number $\ell$, either $T$
  has a leaf whose first label is contained in
  $B$ or $T$ has an inner node of distance exactly $\ell$ from the root
  whose first label is contained in $B$. We show the latter by
  induction on $\ell$.

  The claim obviously holds for $\ell=0$. So assume that $T$ contains a
  node $t$ at distance $\ell$ from the root of $T$ whose first label, say
  $B'$, is a subset of $B$. 
  If $t$ is a leaf node of $T$, then the
  claim is shown. Otherwise, there is an assignment $\assgn{\tau}{B'}{D}$
  such that $\cspi[\tau]$ is not closed under any
  operation from $P$. Let $\assgn{\tau^*}{B}{D}$ be any
  assignment of the variables in $B$ that agrees with $\tau$ on the
  variables in $B'$ and let $\phi \in P$ be such that $\cspi[\tau^*]$
  is closed under $\phi$. Because $B$ is a strong $P$-backdoor set,
  the polymorphism $\phi$ clearly exists. By definition of the search
  tree $T$, $t$ has a child $t'$ for every variable $v$ in the scope
  of some constraint $c \in C[\tau]$ that is not closed under
  $\phi$. We claim that $V(c) \cap B \neq \emptyset$ and hence $t$ has
  a child, whose first label is a subset of $B$, as required. Indeed,
  suppose not. Then $c \in C[\tau^*]$ a contradiction to our
  assumption that $\cspi[\tau^*]$ is closed under $\phi$.
  This concludes our proof concerning the correctness of the algorithm.

  The running time of the algorithm is obtained as
  follows. Let $T$ be a search tree obtained by the
  algorithm. Then the running time of the depth-bounded search tree
  algorithm is $O(|V(T)|)$ times the maximum time that is spent on any
  node of $T$. Since the number of children of any node of $T$ is
  bounded by $|P|\arity(\cspi)$ (recall that $\arity(\cspi)$ denotes the maximum 
  arity of any constraint of $\cspi$) and the
  longest path from the root of $T$ to some
  leaf of $T$ is bounded by $k+1$, we obtain that $|V(T)| \leq
  O({(|P|\arity(\cspi))}^{k+1})$. Furthermore, the time
  required for any node $t$ of $T$
  is at most
  $O({|D|}^{k}\textup{comp\_rest}(\cspi,\tau)|C(\cspi[\tau])||P|\textup{check\_poly}(c,\phi))$, 
  where $\textup{comp\_rest}(\cspi,\tau)$ is the time required to compute
  $\cspi[\tau]$ for some assignment $\tau$ of at most $k$ variables
  and $\textup{check\_poly}(c,\phi)$ is the time required to
  check whether a constraint $c$ of $\cspi[\tau]$ preserves the
  operation $\phi \in P$. Observe that
  $\textup{comp\_rest}(\cspi,\tau)$ and $|C(\cspi[\tau])|$ are
  polynomial in the input size. The same holds for
  $\textup{check\_poly}(c,\phi)$, because $\phi$ is a $c_{\prop{P}}$-ary
  operation.
  Now, the total running time required by the algorithm is the time
  required to compute the set $P$ plus the time required to compute
  $T$. Putting everything together, we obtain
  $O((|P|\arity(\cspi))^{k+1}|D|^{k}|P|n^{O(1)})=O((|P|\arity(\cspi)|D|)^{k+2}n^{O(1)})=O(({|D|}^{{|D|}^{c_{\prop{P}}}}\arity(\cspi)|D|)^{k+2}n^{O(1)})$, r
  as the total running time of the algorithm, where $n$ denotes the
  input size of the CSP instance.
  This
  shows that \STBDET{\CCC_{\prop{P}}} is fixed-parameter tractable
  parameterized by $k$, $\arity(\cspi)$, and $|D|$.
\end{proof}
Because of Proposition~\ref{pro:schaefer-nice}, we obtain the following.
\begin{COR}
  Let $\CCC$ be a base class consisting of the union of some of the classes
  $\MINMAX$, $\MAJ$, $\AFF$, and $\MAL$.
  Then
  \STBDET{\CCC} is fixed-parameter tractable for the
  combined parameter size of the backdoor set, size of the domain, and
  the maximum arity of the given CSP instance.
\end{COR}

We are now ready to provide analogous results for weak backdoor
sets. The proof is very similar to the case of strong backdoor sets,
however, it becomes slightly simpler, because we do not need to
consider different operations for different assignments (since
there is only one assignment, required for a weak backdoor set). We
can hence branch on the possible operations before we start the
depth-bounded search tree procedure, which also results in a slight
improvement in the running time of the algorithm.
\begin{THE}\label{the:fpt-csp-weak}
  Let $\prop{P}$ be a tractable polymorphism property. Then
  \WEBDET{\CCC_{\prop{P}}} is fixed-parameter tractable for the
  combined parameter size of the backdoor set, size of the domain, and
  the maximum arity of the given CSP instance.
\end{THE}
\begin{proof}
  Let $\prop{P}$ be a tractable property,  and let
  $\tuple{\cspi,k}$ with $\cspi=\tuple{V,D,C}$ be an instance of
  \WEBDET{\CCC_{\prop{P}}}. Let $P$ be the set of all operations on
  $D$ that have property $\prop{P}$. Then, $P$ can be constructed in
  fpt-time with respect to the size of the domain
  , because there are at most
  $|D|^{|D|^{c_{\prop{P}}}}$ $c_{\prop{P}}$-ary operations on $D$
  and for each of them we can test in polynomial time, $|D|^{O(c_{\prop{P}})}$, whether it
  satisfies property $\prop{P}$.  The algorithm uses a depth-bounded
  search tree approach to find a weak $\{\phi\}$-backdoor set
  of size at most $k$ for every $\phi \in P$.

  The algorithm works by constructing a search tree $T_\phi$ for every
  operation $\phi \in P$. In $T_\phi$
  every node is labeled by a set $B$ of at most $k$ variables of
  $V$. Additionally, every leaf node has a second label,
  which is either \textsc{Yes} or \textsc{No}. 
  $T_\phi$ is defined inductively as follows. The root of $T_\phi$ is labeled by
  the empty set. Furthermore, if $t$ is a node of $T_\phi$,
  whose first label is $B$, then the
  children of $t$ in $T_\phi$ are obtained as follows. If there is an
  assignment $\assgn{\tau}{B}{D}$ such that $\cspi[\tau]$ is closed
  under $\phi$ and $\cspi[\tau]$ has a solution
  , then $B$ is a
  weak $P$-backdoor set of size at most $k$, and hence $t$ becomes a
  leaf node, whose second label is \textsc{Yes}. Otherwise, i.e., for every
  assignment $\assgn{\tau}{B}{D}$ either
  $\cspi[\tau]$ is not closed under $\phi$ or
  $\cspi[\tau]$ has no solution, we
  consider two cases: (1) $|B|=k$, then $t$ becomes a leaf node, whose
  second label is \textsc{No}, and (2) $|B|<k$,
  then for every assignment $\assgn{\tau}{B}{D}$ 
  $t$ has the following children:
  for every every variable $v$ in
  the scope of some constraint $c \in \cspi[\tau]$ that is not closed
  under $\phi$, $t$ has a child whose first label is $B \cup \{v\}$.

  If there is a $\phi \in P$ such $T_\phi$ has a leaf node, whose
  second label is \textsc{Yes}, then the
  algorithm returns the first label of that leaf node. Otherwise the
  algorithm returns \textsc{No}. This completes the description of the
  algorithm.

  We now show the correctness of the algorithm. First, suppose 
  there is a $\phi \in P$ such that the search tree $T_\phi$ built by
  the algorithm has a leaf node $t$ whose
  second label is \textsc{Yes}. Here, the algorithm returns the first label, say
  $B$ of $t$. By the construction of $T_\phi$, we obtain that $|B|\leq k$
  and there is an assignment $\assgn{\tau}{B}{D}$ such that $\cspi[\tau]$ is
  closed $\phi$ and $\cspi[\tau]$ has a solution, as required.

  Now consider the case where the algorithm returns
  \textsc{No}. We need to show that there is no 
  weak $P$\hy backdoor set of size at most $k$ for $\cspi$.
  Assume, for the sake of contradiction that such a set $B$
  exists. Because $B$ is a weak $P$\hy backdoor set for $\cspi$ there
  must exist an operation $\phi \in P$ and an assignment
  $\assgn{\tau}{B}{D}$ such that $\cspi[\tau]$ is closed under $\phi$
  and $\cspi[\tau]$ has a solution. 

  Observe that if $T_\phi$ has a leaf node
  $t$ whose first label is a set $B'$ with $B' \subseteq B$,
  then the second label of $t$ must be \textsc{Yes}. This is because,
  either $|B'|<k$ in which case the second label of $t$ must be
  \textsc{Yes}, or $|B'|=k$ in which case $B'=B$ and by the definition
  of $B$ it follows that the second label of $t$ must be \textsc{Yes}.

  It hence remains to show that $T_\phi$ has a leaf node whose first label is a set $B'$
  with $B' \subseteq B$. This will complete the proof about the
  correctness of the algorithm. We will show a slightly stronger
  statement, namely, that for every natural number $\ell$, either $T_\phi$
  has a leaf whose first label is contained in
  $B$ or $T_\phi$ has an inner node of distance exactly $\ell$ from the root
  whose first label is contained in $B$. We show the latter by
  induction on $\ell$.

  The claim obviously holds for $\ell=0$. So assume that $T_\phi$ contains a
  node $t$ at distance $\ell$ from the root of $T_\phi$ whose first label, say
  $B'$, is a subset of $B$. 
  If $t$ is a leaf node of $T_\phi$, then the
  claim is shown. Otherwise, for every assignment $\assgn{\tau}{B'}{D}$
  either $\cspi[\tau]$ is not closed under $\phi$ or $\cspi[\tau]$ has no solution. 
  Let $\assgn{\tau'}{B'}{D}$ be the restriction of $\tau$ to $B'$.
  Because $\cspi[\tau]$ and hence also $\cspi[\tau']$ has a solution,
  we obtain that $\cspi[\tau']$ is not closed under the operation $\phi$.
  By the definition of the search tree $T_\phi$, it holds that $t$ has a
  child $t'$ whose first label is $B' \cup \{v\}$, for every variable
  $v$ in the scope of some constraint $c \in\cspi[\tau']$ that is not
  closed under $\phi$. We claim that $V(c) \cap B \neq \emptyset$ and
  hence $t$ has a child, whose first label is a subset of $B$, as
  required. Indeed, suppose not. Then $c \in \cspi[\tau]$ a
  contradiction to our assumption that $\cspi[\tau]$ is closed under
  $\phi$. This concludes our proof concerning the correctness of the algorithm.

  The running time of the algorithm is obtained as
  follows. Let $T_\phi$ be a search tree obtained by the
  algorithm. Then the running time of the depth-bounded search tree
  algorithm is $O(|V(T_\phi)|)$ times the maximum time that is spent on any
  node of $T_\phi$. Since the number of children of any node of $T_\phi$ is
  bounded by $|D|^{k}\arity(\cspi)$ (recall that $\arity(\cspi)$ denotes the maximum 
  arity of any constraint of $\cspi$) and the
  longest path from the root of $T_\phi$ to some
  leaf of $T_\phi$ is bounded by $k+1$, we obtain that $|V(T_\phi)| \leq
  O({(|D|^{k}\arity(\cspi))}^{k+1})$. Furthermore, the time
  required for any node $t$ of $T$
  is at most
  $O({|D|}^{k}\textup{comp\_rest}(\cspi,\tau)|C(\cspi[\tau])|\textup{check\_poly}(c,\phi)\textup{check\_sol}(\cspi[\tau],\phi))$, 
  where $\textup{comp\_rest}(\cspi,\tau)$ is the time required to compute
  $\cspi[\tau]$ for some assignment $\tau$ of at most $k$ variables,
  $\textup{check\_poly}(c,\phi)$ is the time required to
  check whether a constraint $c$ of $\cspi[\tau]$ preserves the
  polymorphism $\phi$, and
  $\textup{check\_sol}(\cspi[\tau],\phi)$ 
  is the time required to solve $\cspi[\tau]$ given that it is
  closed under the operation $\phi$. Observe that
  $\textup{comp\_rest}(\cspi,\tau)$ and $|C(\cspi[\tau])|$ are
  polynomial in the input size. The same holds for
  $\textup{check\_poly}(c,\phi)$ and
  $\textup{check\_sol}(\cspi[\tau],\phi)$, because $\phi$ is a
  $c_{\prop{P}}$-ary tractable polymorphism.
  Now, the total running time required by the algorithm is the time
  required to compute the set $P$ plus the time required to compute
  $T_\phi$ for every $\phi \in P$. Putting everything together, we obtain
  $O(|P|(|D|^k\arity(\cspi))^{k+1}|D|^{k}n^{O(1)})=O(|P|(|D|^k\arity(\cspi))^{k+1}n^{O(1)})=O({|D|}^{{|D|}^{c_{\prop{P}}}}(|D|^k\arity(\cspi))^{k+2}n^{O(1)})$,
  as the total running time of the algorithm, where $n$ denotes the
  input size of the CSP instance.
  This
  shows that \WEBDET{\CCC_{\prop{P}}} is fixed-parameter tractable
  parameterized by $k$, $\arity(\cspi)$, and $|D|$.
\end{proof}
Because of Proposition~\ref{pro:schaefer-nice}, we obtain the following.
\begin{COR}
  Let $\CCC$ be a base class consisting of the union of some of the classes
  $\MINMAX$, $\MAJ$, $\AFF$, and $\MAL$.
  Then
  \WEBDET{\CCC} is fixed-parameter tractable for the
  combined parameter size of the backdoor set, size of the domain, and
  the maximum arity of the given CSP instance.
\end{COR}

\section{Hardness of Backdoor Detection for CSP}

In this section we show our parameterized hardness results for
\STBDET{\CCC_{\prop{P}}} and \WEBDET{\CCC_{\prop{P}}}. In particular, we show that
\STBDET{\CCC_{\prop{P}}} and \WEBDET{\CCC_{\prop{P}}} are 
\W[2]\hy hard parameterized by the size of the backdoor set even
for CSP instances of Boolean domain and for CSP instances with arity
two. We start by showing hardness for CSP instances with Boolean domain.
\begin{THE}\label{the:domain}
  Let $\prop{P}$ be a tractable polymorphism property such that all
  operations $\phi$ with $\prop{P}(\phi)$ are idempotent. Then,
  \STBDET{\CCC_{\prop{P}}} and \WEBDET{\CCC_{\prop{P}}} are $\W[2]$\hy hard parameterized by the
  size of the backdoor set, even for CSP instances over the Boolean
  domain.
\end{THE}
\begin{proof}
  We start by introducing what we call ``Boolean barriers'' of
  operations since they form the basis of the proof.
  Let $\phi : D^n \rightarrow D$ be an $n$-ary operation over $D$. 
  We say a set $\lambda$ of $r(\lambda)$-ary tuples over $\{0,1\}$
  is a \emph{Boolean barrier} for $\phi$ if there is a sequence
  $\tuple{t_1,\dotsc, t_{n}}$ of (not necessarily distinct) tuples in
  $\lambda$ such that $\phi(t_1,\dotsc,t_n)\notin \lambda$. 
  We call a Boolean barrier $\lambda$ of $\phi$
  \emph{minimal} if $|\lambda|$ is minimal over all
  Boolean barriers of $\phi$. For an operation $\phi$, we denote by $\lambda(\phi)$ a minimal
  Boolean barrier of $\phi$. 
  For our reduction below, we will employ the fact that every
  operation $\phi$ with $\prop{P}(\phi)$ has a non-empty Boolean
  barrier of finite size. The reason that a Boolean barrier must
  exist is simply because $\phi$ is tractable and if $\phi$ would not
  have a Boolean barrier then every Boolean CSP instance would be
  closed under $\phi$ and thus tractable, which unless $\mP=\NP{}$ is
  not possible. To see that $\lambda(\phi)$ is finite first note that
  $|\lambda(\phi)|$ is at most as large as the arity $c_{\prop{P}}$ of
  $\phi$. Moreover, $r(\lambda)\leq 2^{c_{\prop{P}}}$ because there are
  at most $2^{c_{\prop{P}}}$ distinct $c_{\prop{P}}$-ary tuples over $\{0,1\}$.
  Hence the size of $\lambda(\phi)$ is at most $c_{\prop{P}} \cdot 2^{c_\prop{P}}$.

  We show the theorem via an fpt-reduction from \textsc{Hitting Set}.
  Given an instance $(\mathcal{Q}, U,k)$ of
  \textsc{Hitting Set}, we construct a CSP instance
  $\cspi=\tuple{V,\{0,1\},C}$ such that 
  $\mathcal{Q}$ has a hitting set of size at most $k$ if and only if
  $\cspi$ has a strong $\CCC_{\prop{P}}$-backdoor set of size
  at most $k$. 
  Let $\Lambda$ be the set of all Boolean barriers of
  polymorphisms with property $\prop{P}$, i.e.,
  $\Lambda=\SB \lambda(\phi) \SM \phi \textup{ has
    property }\prop{P} \SE$. Note that the cardinality of
  $\Lambda$ only depends on $c_{\prop{P}}$ and can thus
  be considered to be finite. 
  The
  variables of $\cspi$ are $\SB x_u \SM u \in U \SE \cup \SB
  o_1(\lambda,Q),\dotsc,o_{r(\lambda)}(\lambda,Q) \SM \lambda \in \Lambda \textup{ and }Q \in
  \mathcal{Q} \SE$. 
  Furthermore, for every Boolean barrier $\lambda \in \Lambda$
  with $\lambda=\{t_1,\dotsc,t_n\}$ for some natural
  number $n$, and $Q
  \in \mathcal{Q}$ with $Q=\{u_1,\dotsc,u_{|Q|}\}$, $C$ contains a constraint 
  $R(\lambda,Q)$ with scope $\tuple{o_1(\lambda,Q),\dotsc,o_{r(\lambda)}(\lambda,Q),x_{u_1},\dotsc,x_{u_{|Q|}}}$
  whose relation contains the row
\[t_i[1], \dotsc,t_i[r(\lambda)], \underbrace{\tuple{i\mod 2, \dotsc,i\mod
      2}}_{|Q| \text{times}}
\]
  \noindent for every $i$ in $1 \leq i \leq n$. 
  This completes the construction of $\cspi$. 
  Observe that $\cspi$ is
  satisfiable, because the first rows of every constraint of $\cspi$ are
  pairwise compatible. In particular, set all the variables of the form
  $\{x_u : u \in U\}$ to zero, and set all remaining variables
  according to their first-row constraints --- note that these rows are
  the same for different $Q \neq Q'$ associated with the same
  Boolean barrier $\lambda$, while on the other hand, for different
  Boolean barriers
  $\lambda \neq \lambda'$, the associated variables are different and there is no
  cause for conflict.
  Suppose that $\mathcal{Q}$ has a hitting set $B$ of size at most $k$. We claim
  that $B_u=\SB x_u \SM u \in B \SE$ is a strong
  $\CCC_{\prop{P}}$-backdoor set of
  $\cspi$. Let $\assgn{\tau}{B_u}{D}$ be an assignment of the
  variables in $B$ and $\lambda \in \Lambda$ be a Boolean barrier
  of maximum cardinality over all Boolean barriers in $\Lambda$.
  We claim that $\cspi[\tau]$ is closed under any operation $\phi$
  with property $\prop{P}$ with $\lambda=\lambda(phi)$ and hence $B_u$ is a
  strong $\CCC_{\prop{P}}$-backdoor set of $\cspi$. Because $B$ is a
  hitting set of $\mathcal{Q}$,
  it follows that every relation of $\cspi[\tau]$ contains at most half of the tuples of
  the corresponding relation in $\cspi$.
  Furthermore, because $\phi$ is idempotent (and hence every tuple is
  mapped to itself), it holds
  that $|lambda'|>1$ for every $\lambda'\in \Lambda$. It follows that every relation of
  $\cspi[\tau]$ contains at least one tuple less than 
  the corresponding relation in $\cspi$. Because of the choice of $\lambda$, we obtain that
  every constraint $c \in C[\tau]$ contains less than $|\lambda|$
  tuples, and hence $\cspi[\tau]$ is closed under any operation $\phi$
  with property $\prop{P}$ and $\lambda=\lambda(\phi)$. Hence,
  $B_u=\SB x_u \SM u \in B \SE$ is a strong $\CCC_{\prop{P}}$-backdoor set of
  $\cspi$, as required. 
  Towards showing that $B_u$ is also a weak $\CCC_{\prop{P}}$-backdoor set of
  $\cspi$ consider the assignment $\assgn{\tau}{B_u}{\{0\}}$. It
  follows from the above argumentation that $C[\tau]$ is closed under 
  any operation $\phi$ with property $\prop{P}$ and
  $\lambda=\lambda(\phi)$. 
  Furthermore, $\cspi[\tau]$ is satisfiable, because
  the first row of every constraint
  in $C$ is also contained in every constraint of $\cspi[\tau]$ and these
  rows are pairwise compatible.

  For the reverse direction, suppose that $\cspi$ has a strong
  $\CCC_{\prop{P}}$-backdoor set $B$ of size at most $k$. Because for
  every $Q \in \mathcal{Q}$ it holds that the set of constraints $\SB
  R(\lambda,Q)\SM \lambda \in \Lambda\SE$ is not closed under any operation 
  $\phi$ with property $\prop{P}$, we obtain
  that $B$ has to contain at least one variable from 
  $\bigcup_{\lambda \in \Lambda}V(R(\lambda,Q))$ for every $Q \in \mathcal{Q}$. Since the
  only variables that are shared between $\bigcup_{\lambda \in \Lambda}V(R(\lambda,Q))$
  and $\bigcup_{\lambda \in \Lambda}V(R(\lambda,Q'))$ for distinct $Q,Q' \in \mathcal{Q}$
  are the variables in $\SB x_u \SM u\in U \SE$,
  it follows that $\mathcal{Q}$ has a hitting set of size at most $|B|\leq k$,
  as required.
\end{proof}

Because all min/max, majority, minority and Mal'cev operations can
be defined via tractable properties and are idempotent, we obtain.
\begin{COR}
  For every $\CCC \in \{\MINMAX$, $\MAJ$, $\AFF$, $\MAL\}$,
  \STBDET{\CCC} and \WEBDET{\CCC} are $\W[2]$\hy hard parameterized by the size of the
  backdoor set, even for CSP instances over the Boolean domain.
\end{COR}

In the following we show that hardness also holds if we drop the
restriction on the domain of the CSP instance but instead consider
only CSP instances of arity $2$. 
To do so we need the following lemma.
\begin{LEM}\label{lem:undom-rel}
  For every $\CCC \in \{\MINMAX,\MAJ,\AFF,\MAL\}$
  and every $k\geq 3$,
  there is a $2$-ary CSP instance $\cspi(\CCC,k)$ with $k$ constraints, each containing the all $0$ tuple,
  such that $\cspi(\CCC,k) \notin \CCC$ 
  but for every assignment $\tau$ of at least one variable of
  $\cspi(\CCC,k)$, it holds that $\cspi(\CCC,k)[\tau] \in \CCC$.
\end{LEM}
\begin{proof}
  The proof of this lemma is inspired by the proof of~\cite[Lemma
  3]{BessiereCarbonnelHebrardKatsirelosWalsh13}.
  For every $\CCC$ as in the statement of the lemma, 
  $\cspi(\CCC,k)$ has variables $v_1,\dotsc,v_{2k}$
  and contains a constraint $c_i$ with scope
  $\tuple{v_{2i-1},v_{2i}}$ for every $1 \leq i \leq k$.

  If $\CCC=\MAJ$, the domain of $\cspi(\CCC,k)$ is
  $\{1,\dotsc,3k-4\}$, and for every $1 \leq i \leq k$ the relations
  $R(c_i)$ are defined as follows: $R(c_1)=\{(0,0), (1,3), (1,4), (2,5)\}$, 
  $R(c_2)=\{(0,0), (1,3), (2,4), (1,5)\}$, $R(c_3)=\{(0,0),
  (2,3(k-2)), (1,3(k-2)+1), (2,3(k-2)+2) \}$, and for every $3<i\leq
  k$, $R(c_i)=\{(0,0), (3(i-3),3(i-2)), (3(i-3)+1,3(i-2)+1),
  (3(i-3)+2,3(i-2)+2)\}$. Clearly, the instance is binary and every
  relation contains the all $0$ tuple. We first show that
  $\cspi(\CCC,k)$ is not closed under any majority operation $\phi$.
  From $R(c_1)$ and $R(c_2)$, we obtain $\phi(3,4,5)=3$ and from
  $R(c_3)$, we obtain $\phi(3(k-2), 3(k-2)+1,3(k-2)+2) \in
  \{3(k-2)+1,3(k-2)+2\}$. Finally, from $R(c_i)$ with $3 < i \leq k$,
  we obtain $\phi(3(i-2),3(i-2)+1,3(i-2)+2)=j$ if and only if
  $\phi(3(i-1),3(i-1)+1,3(i-1)+2)=j+3$. This chain implies $f(3,4,5)
  \in \{4,5\}$ in contradiction to $f(3,4,5)=3$. On the other hand,
  for any assignment $\tau$ that sets at least $1$ variable of
  $\cspi(\CCC,k)$, it holds that at least one relation of
  $\cspi(\CCC,k)[\tau]$ contains at most $2$ tuples. This relation is
  then closed under any majority operation and breaks the above
  chain of implications.

  If $\CCC=\MINMAX$, the domain of $\cspi(\CCC,k)$ is
  $\{1,\dotsc,k+1\}$, and for every $1 \leq i \leq k$ the relations
  $R(c_i)$ are defined as follows:
  $R_i=\{(0,0),(i,i+1),(i+1,i),(i+1,i+1)\}$ for every $1 \leq i < k$
  and $R_k=\{(0,0),(1,k+1),(k+1,1),(1,1)\}$. 
  Clearly, the instance is binary and every
  relation contains the all $0$ tuple. Let $\phi$ be a max operation (the proof remains the
  same if $\phi$ is a min operation). 
  Then,
  for every $1 \leq i <k$, the relation $R_i$ implies $i < i+1$, but
  the relation $R_k$ implies $1>k+1$, a contradiction. Hence, the set
  of all relations $R_1,\dotsc,R_k$ is not closed under any max operation.  On the other hand,
  for any assignment $\tau$ that sets at least $1$ variable of
  $\cspi(\CCC,k)$, it holds that at least one relation of
  $\cspi(\CCC,k)[\tau]$ contains at most $2$ tuples that agree on (at
  least) one coordinate. This relation is
  then closed under any max operation and breaks the above
  chain of implications. 

  If $\CCC=\AFF$, the domain of $\cspi(\CCC,k)$ is
  $\{1,\dotsc,3k+2\}$, and for every $1 \leq i \leq k$ the relations
  $R(c_i)$ are defined as follows:
  $R_1=\{(0,0),(1,3),(1,4),(2,5)\}$,
  $R_i=\{(0,0),(3(i-1),3i),(3(i-1)+1,3i+1),(3(i-1)+2,3i+2)\}$ for every $1 <
  i < k$, and $R_k=\{(0,0),(1,3(k-1)),(2,3(k-1)+1),(1,3(k-1)+1)\}$. 
  Clearly, the instance is binary and every
  relation contains the all $0$ tuple. Let $\phi$ be a minority operation.
  Then, $R_1$ implies $f(3,4,5)=5$, which together with the
  relations $R_i$, for every $1 < i <k$ implies
  $f(3i,3i+1,3i+2)=3i+2$. In particular, we obtain
  $f(3(k-1),3(k-1)+1,3(k-1)+2)=3(k-1)+2$, which contradicts relation
  $R_k$. Hence, the set
  of all relations $R_1,\dotsc,R_k$ is not closed under any minority
  operation.
  On the other hand,
  for any assignment $\tau$ that sets at least $1$ variable of
  $\cspi(\CCC,k)$, it holds that at least one relation of
  $\cspi(\CCC,k)[\tau]$ contains at most $2$ tuples. This relation is
  then closed under any minority operation and breaks the above
  chain of implications. If $\CCC=\MAL$, we can use the same
  set of relations.
\end{proof}
\begin{THE}\label{the-arity2}
  For every $\CCC \in \{\MINMAX$, $\MAJ$, $\AFF$, $\MAL\}$,
  \STBDET{\CCC} and \WEBDET{\CCC} are $\W[2]$-hard 
  parameterized by the size of the
  backdoor set even for CSP instances with arity $2$.
\end{THE}
\begin{proof}
  The proof of the theorem is inspired by the proof
  of~\cite[Theorem 6]{BessiereCarbonnelHebrardKatsirelosWalsh13}. 
  We show the theorem via a fpt-reduction from \textsc{Hitting Set}.
  Let $H=\tuple{{\mathcal U},{\mathcal F},k}$ be an instance of
  \textsc{Hitting Set}. 

  We construct a CSP instance $\cspi=\tuple{V,D,C}$ of arity $2$ such that 
  $H$ has a hitting set of size at most $k$ if and only if
  $\cspi$ has a strong $\CCC$-backdoor set of size at most $k$. The
  variables of $\cspi$ are $\SB x_u \SM u \in \mathcal{U} \SE \cup \SB
  y_F^i \SM F \in \mathcal{F} \textup{ and } 1 \leq i \leq |F|
  \SE$. Furthermore, for every $F \in \mathcal{F}$ with
  $F=\{u_1,\dotsc,u_{|F|}\}$, let
  $\cspi(F)$ be the CSP instance obtained from $\cspi(\CCC,|F|)$
  (see Lemma~\ref{lem:undom-rel}), where
  for every $i$ with $1 \leq i \leq |F|$, we replace the
  scope of the constraint $c_i$ with $\tuple{x_{u_i},y_F^i}$ and adapt
  $\cspi(\CCC,|F|)$ in such a way that
  for distinct $F,F' \in \mathcal{F}$, $\cspi(\CCC,|F|)$ and
  $\cspi(\CCC,|F'|)$ have no common domain value.
  Then, the constraints of $\cspi$ is the union of the constraints of
  $\cspi(F)$ over every $F \in \mathcal{F}$. This completes the
  construction of $\cspi$. Clearly, $\cspi$ has arity $2$ and is
  satisfiable, e.g., by the all $0$ assignment.

  Suppose that $H$ has a hitting set $B$ of size at most $k$. We claim
  that $B_u=\SB x_u \SM u \in B \SE$ is a strong
  $\CCC$-backdoor set of
  $\cspi$. We start by showing that $B_u$ is a strong
  $\CCC$-backdoor set of
  $\cspi$. Let $\assgn{\tau}{B_u}{D}$ be an assignment of the
  variables in $B$. Then, it follows from Lemma~\ref{lem:undom-rel}
  that the constraints of $\cspi(F)[\tau]$ are in
  $\CCC$. In particular,
  there is an operation $\phi$ (which is either min, max, majority,
  minority or Mal'cev, depending on $\CCC$) such that the constraints of
  $\cspi(F)[\tau]$ are closed under $\phi$. Because for every distinct
  $F,F' \in \mathcal{F}$, $\cspi(F)[\tau]$ and $\cspi(F')[\tau]$ do
  not have a common domain value, it follows that there is an operation
  (which is a combination of the operations for each
  $\cspi(F)[\tau])$) such that $\cspi[\tau]$ is closed under $\phi$
  and $\phi$ is a min/max, majority, minority, or Mal'cev operation
  (depending on $\CCC$). Hence, $\cspi[\tau] \in \CCC$,
  as required.
  Towards showing that $B_u$ is also a weak $\CCC$-backdoor set of
  $C$ consider the assignment $\assgn{\tau}{B_u}{\{0\}}$. It
  follows from the above argumentation that $\cspi[\tau] \in \CCC$. 
  Furthermore, $\cspi[\tau]$ is satisfiable, because
  every constraint of $\cspi[\tau]$ still contains the all $0$ tuple.

  For the reverse direction, suppose that $\cspi$ has a strong
  $\CCC$-backdoor set $B$ of size at most $k$. Because for
  every $F \in \mathcal{F}$ it holds that the set of constraints of
  $\cspi(F)$ is not
  in $\CCC$, we obtain
  that $B$ has to contain at least $1$ variable from 
  $V(\cspi(F))$ for every $F \in \mathcal{F}$. Since the
  only variables that are shared between $V(\cspi(F))$
  and $V(\cspi(F'))$ for distinct $F,F' \in \mathcal{F}$
  are the variables in $\SB x_u \SM u\in \mathcal{U} \SE$,
  it follows that $H$ has a hitting set of size at most $|B|\leq k$,
  as required.
\end{proof}

\section{Comparison to Partition-Backdoors}

In this section we draw a comparison between our notion of
backdoor sets, and the approach recently introduced
by~\citex{BessiereCarbonnelHebrardKatsirelosWalsh13},
which
we will call the \emph{partition backdoor set} approach. The main
difference is that the partition backdoor set approach
considers a backdoor with respect to a subset of constraints that is
tractable (because it is closed under some tractable polymorphism).
In particular, \citex{BessiereCarbonnelHebrardKatsirelosWalsh13} define partition backdoor sets as follows.
\begin{DEF}
  Let $\prop{P}$ be a tractable polymorphism property.
  A $\prop{P}$-partition backdoor of a CSP instance $\cspi=\tuple{V,D,C}$ is 
  a set $B$ of variables, such that there is a partition of $C$ into
  $C_1$ and $C_2$ for which the following holds: 
  \begin{itemize}
  \item $\tuple{V,D,C_2} \in \CCC_{\prop{P}}$ , and
  \item If $\prop{P}$ is characterized using only idempotent
    polymorphisms, then 
    $B$ is the set of all variables in the scope of the constraints of $C_2$,
    while if $\prop{P}$ is characterized using only conservative
    polymorphisms, then the variables of 
    $B$ are a vertex cover of the primal graph of $C_1$. 
  \end{itemize}
\end{DEF}
Note that the primal graph of a set $C$ of constraints is the graph, whose
vertices are all the variables appearing in the scope of the
constraints in $C$ and that has an edge between two variables if and
only if they appear together in the scope of at least one constraint
in $C$.

\citex{BessiereCarbonnelHebrardKatsirelosWalsh13} show that given a CSP instance and a partition $(C_1,C_2)$ as
defined above, then the detection and evaluation of
$\prop{P}$-partition backdoor sets are fixed-parameter tractable
\begin{itemize}
\item in the size of the domain and the number of variables in 
  $C_1$, if $\prop{P}$ is characterized using only idempotent polymorphisms;
\item in the size of the domain and the vertex cover of the primal
  graph of $C_1$, if $\prop{P}$ is characterized using only conservative
  polymorphisms.
\end{itemize}
Observe that the above tractability results only hold
when the partition $(C_1,C_2)$ is given as
a part of the input, if it is not then the detection of partition
backdoors is fixed-parameter intractable (for most
combinations of parameters).

Apart from the complexity of detecting and evaluating backdoors it is
also important that the backdoor sets are as small as possible over a
wide variety of CSP instances. In the following we will observe that our
backdoor sets are always at most as large as partition backdoors and
we exhibit CSP instances where our backdoors have size one
but the size of a smallest partition backdoor set can be arbitrary large.
\begin{PRO}
  Let $\prop{P}$ be a tractable polymorphism property. Then, for every CSP instance
  $\cspi$ the size of a smallest
  strong $\prop{P}$-backdoor set of $\cspi$ is at most the size of a smallest
  $\prop{P}$-partition backdoor set of $\cspi$.
\end{PRO}
\begin{proof}
  This follows from the observation that any $\prop{P}$-partition
  backdoor set is also a strong $\prop{P}$-backdoor set.
\end{proof}
\begin{PRO}
  For any tractable property $\prop{P}$ and any natural
  number $n$
  , there is a CSP instance that has a
  strong/weak $\CCC_{\prop{P}}$-backdoor set of size one,
  but the size of 
  a smallest $\prop{P}$-partition backdoor set is $\Omega(n)$.
\end{PRO}
\begin{proof}
  Let $\prop{P}$ be a tractable property and let $\cspi=\tuple{V,D,C}$ be a
  CSP instance such that:
  \begin{itemize}
  \item no constraint of $\cspi$ is closed under a polymorphism with
    property $\prop{P}$;
  \item all constraints of $\cspi$ share a common variable, say $x$,
    such that for any assignment $\tau$ of $x$, the CSP instance
    $\cspi[\tau]$ is closed under some polymorphisms with property
    $\prop{P}$.
  \end{itemize}
  It is straightforward to check that for any number of variables $n$,
  such a CSP instance using $n$ variables can be
  easily constructed for every tractable property $\prop{P}$ that contains
  at least one operation. Furthermore, the variable $x$ is a strong
  $\CCC_{\prop{P}}$-backdoor set of $\cspi$. However, because no
  constraint of $\cspi$ is closed under any operation from
  $\prop{P}$ it follows that every partition backdoor has to contain
  at least all but $1$ variable from every constraint of $\cspi$.
\end{proof}

\section{Summary}

We have introduced heterogeneous base classes and have shown that
strong backdoor sets into such classes can be considerably smaller
than strong backdoor sets into homogeneous base classes; nevertheless,
the detection of strong backdoor sets into heterogeneous base classes is
still fixed-parameter tractable in many natural cases. Hence, our
results push the tractability boundary considerably further. Our
theoretical evaluation entails hardness results that establish the
limits for tractability. 

Our main result in the context of SAT (Theorem~\ref{the:sat-dico})
provides a complete complexity classification for strong backdoor set detection to heterogeneous base classes
composed of the well-known Schaefer classes.
We observe that the hardness results rely on clauses
of unbounded length and show that if the maximum length of the clauses is taken 
as an additional parameter then 
strong as well as weak backdoor set detection into any combination of the Schaefer classes
is fixed-parameter tractable (Theorem~\ref{the:sat-clause-sb}).

Our main result for CSP backdoors (Theorem~\ref{the:fpt-csp}) establishes
fixed-parameter tractability for a wide range of heterogeneous base classes composed of
possibly infinitely many tractable polymorphisms.
In particular, we show that the detection of strong as well as weak backdoor sets is
fixed-parameter tractable for the combined parameter backdoor size,
domain size, and the maximum arity of constraints.  In fact, this
result entails \emph{heterogeneous} base classes, as different
instantiations of the backdoor variables can lead to reduced instances
that are closed under different polymorphisms (even polymorphisms of
different type). We complement our main result with hardness results
that show that we lose fixed-parameter tractability when we
omit either domain size or the maximum arity of constraints from
the parameter (Theorems~\ref{the:domain} and \ref{the-arity2}).

It would be interesting to extend our line of research
to the study of backdoor sets into heterogeneous base classes composed
of homogeneous classes defined by e.g., structural restrictions such as bounded treewidth~\cite{GaspersSzeider13f,SamerSzeider10a} or bounded hyper-treewidth~\cite{GottlobLeoneScarcello02}, and base classes defined by ``hybrid'' concepts like forbidden
patterns~\cite{CohenCooperCreedMarxSalamon12}, in contrast to the Schaefer classes for SAT, and polymorphism-based classes for CSP. Another promising direction of
future research is to parameterize backdoor sets not by their size
but by structural properties of the backdoor set and how it interacts
with the rest of the instance, similar to the parameters that have
been considered for
modulators into graph classes~\cite{EibenGanianSzeider15b,EibenGanianSzeider15}; some first
results in this direction have recently been obtained~\cite{GanianRamanujanSzeider16b}.

\section*{Acknowledgments}

Gaspers is the recipient of an Australian Research Council Discovery Early
Career Researcher Award (project number DE120101761). NICTA is funded by the
Australian Government as represented by the Department of Broadband,
Communications and the Digital Economy and the Australian Research Council
through the ICT Centre of Excellence program. Misra is supported by the INSPIRE
Faculty Scheme, DST India (project DSTO-1209). Ordyniak acknowledges support
from the Employment of Newly Graduated Doctors of Science for Scientific
Excellence (CZ.1.07/2.3.00/30.0009). Ordyniak and Szeider were supported by the
European Research Council, grant reference 239962 (COMPLEX REASON) and the
Austrian Science Funds (FWF), project I836-N23 (Algorithms and Complexity of
Constraint Languages) and project P26696 (Exploiting New Types of Structure for Fixed Parameter Tractability). \v{Z}ivn\'y is supported by a Royal Society University Research Fellowship and EPSRC grant EP/L021226/1.



\end{document}